\documentclass{article}
\usepackage[round]{natbib}
\usepackage[utf8]{inputenc} % allow utf-8 input
\usepackage[T1]{fontenc}    % use 8-bit T1 fonts
\usepackage{hyperref}       % hyperlinks
\usepackage{url}            % simple URL typesetting
\usepackage{booktabs}       % professional-quality tables
\usepackage{amsfonts}       % blackboard math symbols
\usepackage{nicefrac}       % compact symbols for 1/2, etc.
\usepackage{microtype}      % microtypography
\usepackage{xcolor}         % colors
\usepackage[inline, shortlabels]{enumitem}

\usepackage{makecell}
\usepackage{colortbl}

\definecolor{linkcolor}{RGB}{83,83,182}
\definecolor{citecolor}{RGB}{128,0,128}
\hypersetup{
    colorlinks=true,
    citecolor=citecolor,
    linkcolor=linkcolor,
    urlcolor=linkcolor
}

\usepackage{graphicx}
\usepackage{subfigure}
\usepackage{amsfonts}
\usepackage{amsmath,amssymb,amsthm}
\usepackage{siunitx}
\usepackage{algorithm}
\usepackage{algorithmic}
\usepackage[titlenumbered,linesnumbered,ruled,noend,algo2e]{algorithm2e}

\SetCommentSty{mycommfont}
\SetEndCharOfAlgoLine{}
\usepackage{enumitem}
\usepackage{stfloats}
\usepackage{pifont}% http://ctan.org/pkg/pifont
\newcommand{\cmark}{\textcolor{olive}{\ding{51}}}%
\newcommand{\xmark}{\textcolor{red}{\ding{55}}}%

\newcommand{\cP}{\mathcal{P}}
\newcommand{\cS}{\mathcal{S}}
\newcommand{\cJ}{\mathcal{J}}
\newcommand{\cC}{\mathcal{C}}
\newcommand{\bigo}{\mathcal{O}}
\newcommand{\bbR}{\mathbb{R}}
\newcommand{\bbN}{\mathbb{N}}
\newcommand{\normin}[1]{ \lVert {#1} \rVert}
\newcommand{\norm}[1]{ \left\lVert {#1} \right\rVert}
\newcommand{\abs}[1]{ \lvert {#1} \rvert}
\newcommand{\eqdef}{\triangleq}
\newcommand{\prox}{\mathrm{prox}}
\newcommand{\pluseq}{\mathrel{+}=}

\newcommand{\diveq}{\mathrel{/}=}
\DeclareMathOperator{\dist}{dist}

\DeclareMathOperator{\gsupp}{gsupp}

\DeclareMathOperator{\Id}{Id}

\DeclareMathOperator{\Span}{Span}

\usepackage[nameinlink]{cleveref}
\DeclareMathOperator*{\argmin}{argmin\,}
\DeclareMathOperator*{\argmax}{argmax\,}
\DeclareMathOperator{\MCP}{MCP}
\usepackage{aliascnt}
\usepackage{thmtools,thm-restate}
\newtheorem{theorem}{Theorem}
\newtheorem{assumption}[theorem]{Assumption}
\newtheorem{definition}[theorem]{Definition}
\newtheorem{lemma}[theorem]{Lemma}
\newtheorem{proposition}[theorem]{Proposition}

\newtheorem{example}{Example}
\newtheorem{remark}[theorem]{Remark}

\newcommand{\ie}{{\em i.e.,~}}

\newaliascnt{problem}{equation}
\aliascntresetthe{problem}
\creflabelformat{problem}{#2\textup{(#1)}#3}
\makeatletter

\def\endproblem{\eqno \hbox{\@eqnnum}$$\@ignoretrue}
\makeatother
\crefname{model}{Model}{Models}
\Crefname{problem}{Problem}{Problems}
\crefname{problem}{Pb.}{Pbs.}
\crefname{algorithm}{Algorithm}{Algorithms}
\crefname{figure}{Figure}{Figures}
\crefname{proposition}{Proposition}{Propositions}
\crefname{appendix}{Appendix}{Appendix}
\crefname{assumption}{Assumption}{Assumptions}

\usepackage{enumitem}
\newlist{lemmaenum}{enumerate}{1} % also creates a counter called 'lemmaenum'
\setlist[lemmaenum]{label=\emph{\roman*)}, ref=\thetheorem~\emph{\roman*)}}
\crefalias{lemmaenumi}{lemma}

\newlist{thmenum}{enumerate}{1} % also creates a counter called 'lemmaenum'
\setlist[thmenum]{label=\emph{\roman*)}, ref=\thetheorem~\emph{\roman*)}}
\crefalias{thmenumi}{theorem}

\usepackage[final]{neurips_2022}

\title{Beyond L1: Faster and Better Sparse Models with skglm}

\author{%
Quentin Bertrand
  \\
  Mila \& UdeM, Canada
  \\
  \texttt{quentin.bertrand@mila.quebec} \\
  \And
  Quentin Klopfenstein \\
  Luxembourg Centre for Systems Biomedicine
  \\
   University of Luxembourg\\
  Esch-sur-Alzette, Luxembourg  \\
  \AND
  Pierre-Antoine Bannier \\
  Independent Researcher \\
  \And
  Gauthier Gidel \\
  Mila \& UdeM, Canada \\
  Canada CIFAR AI Chair \\
  \And
  Mathurin Massias \\
  Univ. Lyon, Inria, CNRS, ENS de Lyon,
  \\
  UCB Lyon 1, LIP UMR 5668, F-69342 \\
  Lyon, France \\
}

\begin{document}

\maketitle

\begin{abstract}
  We propose a new fast algorithm to estimate any sparse generalized linear model with convex or non-convex separable penalties.
  Our algorithm is able to solve problems with millions of samples and features in seconds, by relying on coordinate descent, working sets and Anderson acceleration.
  It handles previously unaddressed models, and
  is extensively shown to improve state-of-art algorithms.
  We release \texttt{skglm}, a flexible, \texttt{scikit-learn} compatible package, which easily handles customized datafits and penalties.
\end{abstract}

%! TEX root=../main.tex

%%%%%%%%%%%%%%%%%%%%%%%%%%%%%%%%%%%%%%%%%%%%%%%%%%
\section{Introduction}
%%%%%%%%%%%%%%%%%%%%%%%%%%%%%%%%%%%%%%%%%%%%%%%%%%
%
Sparse generalized linear models play a central role in modern machine learning and signal processing.
The Lasso \citep{Tibshirani96} and its derivatives
\citep{Zou_Hastie05,Ng04,Candes_Wakin_Boyd2008,Simon_Friedman_Hastie_Tibshirani12}
have found numerous successful applications to large scale tasks in genomics \citep{ghosh2005classification}, vision \citep{Mairal}, or neurosciences \citep{Strohmeier_Bekhti_Haueisen_Gramfort2016}.
This impact was made possible by two key factors: efficient algorithms and software implementations.

State-of-the-art algorithms for ``smooth + non-smooth separable'' problems predominantly rely on coordinate descent (CD, \citealt{tseng2009coordinate,Nesterov2012}),
which, when it can be applied, is more efficient than full gradient methods \citep[Sec. 6.1]{Richtarik_Takac2014}.
Coordinate descent can even be improved with Nesterov-like acceleration, to obtain improved convergence rates \citep{Lin_Lu_Xiao_2014,Fercoq_Richtarik2015}.
However, these better rates may fail to reflect in practical accelerations.
On the contrary, \citet{Bertrand_Massias2020} relied on Anderson acceleration \citep{Anderson65} to provide both better rates and practical acceleration for coordinate descent.

Even with efficient algorithms such as coordinate descent,
the practical use of sparsity hits a computational barrier for problems with more than millions of features \citep{lemorvan2018whinter}.
Multiple techniques have been proposed to make coordinate descent scale to  huge problems.
Notably, algorithms can be accelerated by reducing the number of variables to optimize over, using screening rules or working sets.
Screening rules discard features from the problem in advance (\citealt{Ghaoui_Viallon_Rabbani2010,Bonnefoy_Emiya_Ralaivola_Gribonval15})
or dynamically \citep{Fercoq_Gramfort_Salmon2015,Ndiaye_Fercoq_Gramfort_Salmon17}.
On the other side, working sets \citep{Johnson_Guestrin15,Massias_Gramfort_Salmon2018} iteratively solve larger subproblems and progressively include variables identified as relevant.

For the Lasso and a few convex models, coordinate descent has been broadly disseminated to practitioners in off-the-shelf packages such as \texttt{glmnet} \citep{Friedman_Hastie_Hofling_Tibshirani07} or \texttt{scikit-learn} \citep{Pedregosa_etal11}. %or \texttt{lightning} \citep{Blondel_Pedregosa2016}.
More recently, \texttt{celer}, a state-of-the-art convex working set algorithm \citep{Massias_Vaiter_Salmon_Gramfort2019} allowed for successful applications of the Lasso in large scale  problems in medicine \citep{Reidenbach_Lal_Slim_Mosafi_Israeli2021,kim2021nonlinear} or seismology \citep{Muir_Zhan2021}.

\begin{figure}[tb]
        \centering
        \includegraphics[width=0.7\columnwidth]{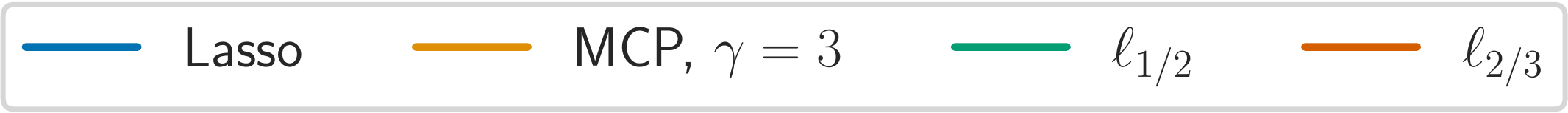}
        \includegraphics[width=0.7\columnwidth]{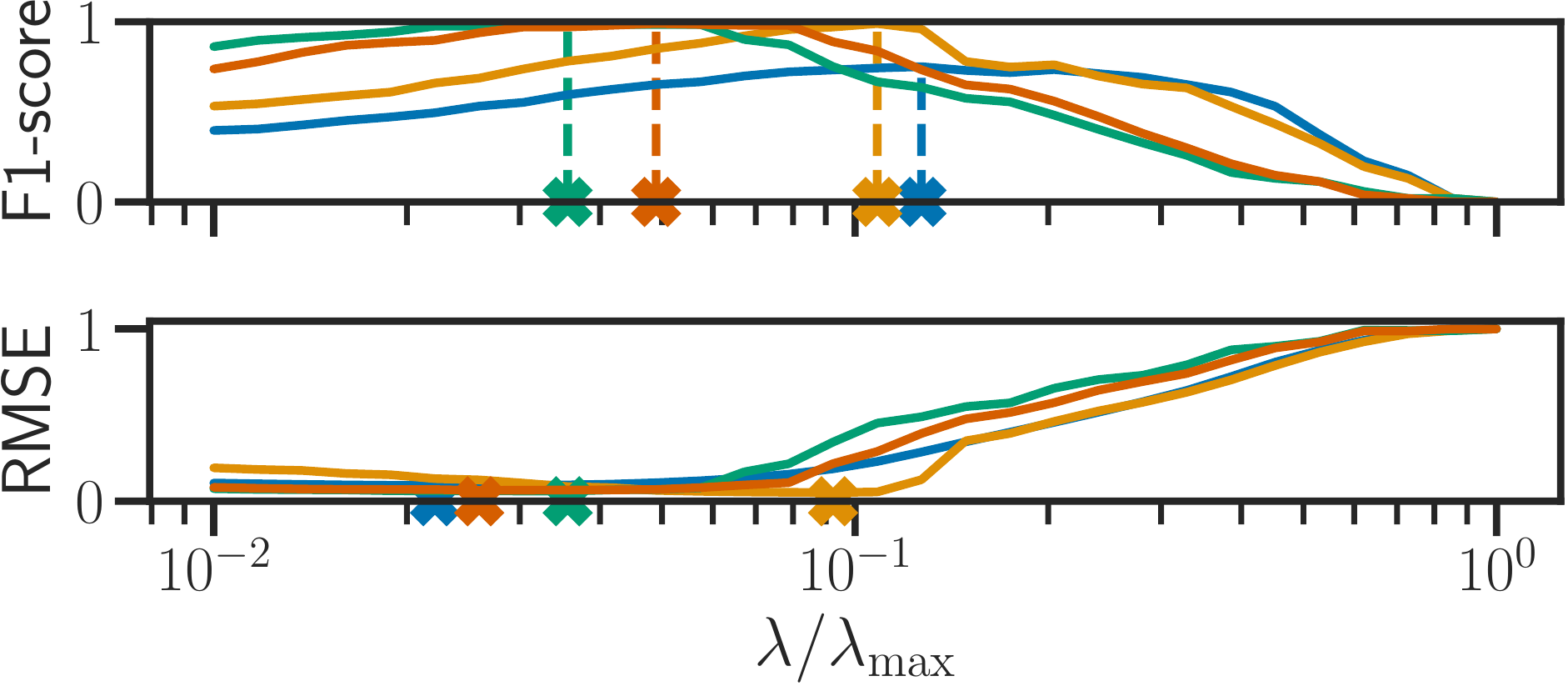}

        \caption{\textbf{Regularization paths computed with our algorithm.}
        Non-convex sparse penalties behave better than the L1 norm. Due to their lower bias, they achieve perfect support recovery, lower prediction error and their optimal regularization strength $\lambda$ in estimation (top) and prediction (bottom) correspond.
        }
        \label{fig:sparse_recovery}
\end{figure}

Yet the Lasso is limited: non-convex sparse models enjoy better theoretical and empirical properties \citep{Breheny_Huang2011,Soubies_BlancFeraud_Aubert2015}.
As illustrated in \Cref{fig:sparse_recovery}, they yield sparser solutions than convex penalties and mitigate the intrinsic Lasso bias.
Yet, they have not so often been applied to huge scale applications.
This is mostly an algorithmic barrier: while coordinate descent can be applied to non-convex penalties \citep{Breheny_Huang2011,Mazumder_Friedman_Hastie2011,Bolte_Sabach_Teboulle2014}, screening rules and working sets are heavily dependent on convexity or quadratic datafits \citep{Rakotomamonjy_Gasso_Salmon19,Rakotomamonjy_Flamary_Gasso_Salmon20}.

In this work, we solve this issue by designing a \textbf{state-of-the-art generic algorithm} to solve a wide range of sparse generalized linear models.
The contributions are the following:
\begin{itemize}[itemsep=0pt,topsep=-1pt]
    \item We propose a non-convex converging working set algorithm relying on Anderson accelerated coordinate descent.
    For a specific class of non-convex penalties, %($\alpha$-semi-convex penalties, \Cref{ass:alpha_semi_cvx})
    we show:
    \begin{enumerate}[itemjoin = \quad,topsep=0pt,parsep=0pt,partopsep=0pt, leftmargin=*,label=(\alph*)]
        % \item Having interpretable scores depicting a more fined grained picture of the characteristic of each player.
        \item Convergence of the proposed working set algorithm (\Cref{prop:conv_ws}).
        \item Support identification of coordinate descent (\Cref{prop:finite_identification}).
        \item Local convergence rates for the Anderson extrapolation (\Cref{prop:acc_lin_conv}).
    \end{enumerate}
    % To our knowledge this is the first convergence rate analysis of
    \item We provide an extensive experimental comparison and
    we show state-of-the-art improvements on a wide range of convex and non-convex problems.
    In addition we release an efficient and modular python implementation, with a \texttt{scikit-learn} API, for practitioners to apply non-convex penalties to large scale problems.
\end{itemize}
% On the practical side,
% \footnote{\url{https://github.com/mathurinm/skglm}.}
% of the proposed algorithm.
% This makes it trivial for practitioners to handle any model within our framework, through a \texttt{scikit-learn} API.

%! TEX root=../main.tex

%%%%%%%%%%%%%%%%%%%%%%%%%%%%%%%%
\section{Framework and proposed algorithm}
%%%%%%%%%%%%%%%%%%%%%%%%%%%%%%%%
%%%%%%%%%%%%%%%%%%%%%%%%%%%%%%%%
\subsection{Problem setting}
%%%%%%%%%%%%%%%%%%%%%%%%%%%%%%%%
%
In this paper, we consider problems of the form:
%
% \vspace*{-3mm}
\begin{problem}\label{pb:generic_pb}
    \hat
    \beta
    \in
    \argmin_{\beta \in \bbR^p}
    \Phi(\beta)
    \eqdef
    \underbrace{F(X\beta)}_{\eqdef f(\beta)}
    + \sum_{j=1}^p g_j(\beta_j)
    \enspace ,
\end{problem}
where $F$ is smooth, and the functions $g_j$ are proper and lower semicontinuous but not necessarily convex, whose proximal operator can be computed exactly.
We write $g = \sum_j g_j$.
Instances of \Cref{pb:generic_pb} include convex estimators: the Lasso, the elastic net, the sparse logistic regression, the dual of SVM with hinge loss.
They also include non-convex penalties:
$\ell_{0.5}$ and $\ell_{2/3}$ penalties \citep{Foucart_Lai2009},
the minimax concave penalty (MCP, \citealt{Zhang2010}) or SCAD \citep{Zhang2010}, both with regression and classification losses.
Formally, the assumptions are the following.
\begin{assumption}\label{ass:f}
    $f : \bbR^p \to \bbR$ is convex and differentiable and for all $j \in [p]$, the restriction of $\nabla_j f$ to the $j$-th coordinate is $L_j$-Lipschitz: for all $(x, h) \in \bbR^p \times \bbR$,
    % \begin{align}
        $| \nabla_j f(x + h e_j) - \nabla_j f (x) |
        \leq
        L_j \abs{h}$.
        % \enspace .
    % \end{align}
\end{assumption}
\begin{assumption}\label{ass:g}
    For any $j \in [p], g_j: \bbR \to \bbR$ is proper, closed, and lower bounded. % MM: removed lsc, closed means lsc.
\end{assumption}
Following \citealt{Attouch_Bolte2009,Bolte_Sabach_Teboulle2014}
we focus on finding a critical point of $\Phi$.
\begin{definition}\label{def:critical_pt}
    Using the Fréchet subdifferential \citep{Kruger2003}, a critical point $x \in \bbR^p$ is a point which satisfies
    % \begin{align} \label{eq:optcond}
        $- \nabla f (x) \in \partial g (x) .$
    % \end{align}
\end{definition}
\Cref{ass:f,ass:g} are usual, and, under boundedness of the iterates, ensure convergence of forward-backward and coordinate descent algorithms to a critical point (\citealt[Thm 5.1]{attouch2013convergence}, \citealt[Thm. 3.1]{Bolte_Sabach_Teboulle2014}).
In addition, our work focuses on the case where $g_j$'s present non-differentiability points, leading to the following extended notion of sparsity.
\begin{definition}[{Generalized support}]\label{def:generalized_supp}
    The \emph{generalized support} %$\cS_{\beta} \subseteq [p]$
    of $\beta \in \bbR^p$ is the set of indices $j \in [p]$ such that $g_j$ is differentiable at $\beta_j$:
        %
    % \begin{align*}
        % \cS =
        $\gsupp(\beta) = \{j \in [p]: \partial g_j(\beta_j) \, \mathrm{is \; a \; singleton}\} .$
    % \end{align*}
\end{definition}%
% \begin{example}
%     With $\lambda, \gamma >0$, MCP is defined as:
%     $\MCP_{\lambda, \gamma}(x) \mapsto
%     \begin{cases}
%         \lambda \abs{x} - \frac{x^2}{2 \gamma} \enspace,
%         &\text{if}\: \abs{x} \leq \gamma \lambda \\
%          \frac{1}{2} \gamma \lambda^2 \enspace,
%         &\text{if}\: \abs{x} > \gamma \lambda
%     \end{cases}\enspace .$
% \end{example}
Penalties such as $\ell_1$, $\ell_q$ ($0 < q < 1)$, MCP or SCAD are only not differentiable at 0, and this corresponds to the usual notion of sparsity.
But \Cref{def:generalized_supp} goes beyond sparsity and extends to estimators such as SVM, where $g_j = \iota_{[0, C]}$ and the generalized support is the complement of the support vectors' set $\{j \in [p] : \beta_j = 0 \text{ or } \beta_j = C  \}$.
The generalized support of a critical point is usually of cardinality much smaller than $p$, and its knowledge makes the problem easier and faster to solve.
% \mathurin{Add an illustrative sentence : For the SVM, fixing the value of coordinates at 0 or C reduces the number of variables in the problem... ?}
Our working set algorithm exploits this structure in order to converge faster.
%
%%%%%%%%%%%%%%%%%%%%%%%%%%%%%%%%%%%%%%%%%%%%%%%%
\subsection{Proposed algorithm}
%%%%%%%%%%%%%%%%%%%%%%%%%%%%%%%%%%%%%%%%%%%%%%%%
%
The proposed algorithm exploits two main ideas:
\begin{itemize}[itemsep=-2pt,topsep=0pt]
    \item A working set strategy, able to handle a large class of convex and non-convex penalties (\Cref{alg:skglm}).
    \item An Anderson accelerated coordinate descent for non-convex problems (\Cref{alg:subproblem}).
    The building blocks of \Cref{alg:subproblem}, coordinate descent (\texttt{CD}, \Cref{alg:cd_epoch}) and Anderson extrapolation (\texttt{Anderson}, \Cref{alg:extrapolation}), can be found in \Cref{app:algorithms}.
\end{itemize}

To avoid wasting computation on features outside the generalized support, working set algorithms iteratively select a subset of coordinates deemed important (the \emph{working set}), and solve \Cref{pb:generic_pb} restricted to them.
The key question is thus the notion of \emph{important}  features.
% There are two types of working set approaches in the literature:
% either selecting all the features breaking a given condition (no control on working set size), or selecting a subset of it, controlling the growth of the subproblems solved.
% %
% The first approach is for example used in strong rules \citep{tibshirani2012strong} or active warm start \citep{Ndiaye_Fercoq_Gramfort_Salmon17}:
% the problem is first solved on a subset of features, then if optimality condition \eqref{eq:optcond} is not satisfied, on all the features.
% Note that in this approach only two subproblems are defined, and the last one is the full problem.
% %
% More efficiently, the second approach, used e.g. in \texttt{blitz} \citep{Johnson_Guestrin15} and \texttt{celer} \citep{Massias_Gramfort_Salmon2018}, assigns a score to each feature to rank them (\Cref{alg:skglm}, line 2).
% Then only a subset of these features are selected as the working set (\Cref{alg:skglm}, line 3).
Stemming from \Cref{def:critical_pt}, we rank features by their violation of the optimality condition:
%
% \begin{equation}\label{eq:score_features}
    $\mathrm{score}_j^\partial =  \dist(-\nabla_j f (\beta), \partial g_j (\beta))
    \enspace .$
% \end{equation}
%
For example, the MCP Fréchet subdifferential at 0 is $\partial g_j (0) = [-\lambda, \lambda]$, and the proposed score reads
\begin{align}
    \label{ex-mcp}
    \mathrm{score}_j^\partial =
    \begin{cases}
        \max \{0, |\nabla_j f(\beta)| - \lambda \}
    \quad &\text{if} \quad \beta_j = 0 \enspace,
    \\
    |\nabla_j f(\beta) + \nabla g_j(\beta_j)|
    \quad &\text{otherwise} \enspace.
    \end{cases}
\end{align}%
To control the working set growth, we use $\mathrm{score}_j^\partial$ to rank the features.
Then, with
$n_k = \max(n_{k-1}, 2 \, | \gsupp(\beta^{(t)})| )$
we take the $n_k$ largest of them in the working set, while retaining features currently in the working set.
This growth quickly rises to the unknown size of the generalized support while avoiding overshooting, as backed up by recent theory in \citet{ndiaye2021continuation}.

\begin{proposition}\label{prop:conv_ws}
    Let $\mathcal{W}_t$ be the $t$-th working set.
    Suppose that \Cref{alg:subproblem} converges toward a critical point, and for all $t\geq 0$, $\mathcal{W}_t \subset \mathcal{W}_{t+1}$, then the
    iterates of
    \Cref{alg:skglm} converge towards a critical point of \Cref{pb:generic_pb}.
\end{proposition}

\begin{minipage}[t]{0.47\linewidth}
    \begin{algorithm}[H]
    \SetKwInOut{Init}{init}
    \SetKwInOut{Input}{input}
    \Input{
        $X, \beta \in \bbR^p,
        n_{\mathrm{out}} \in \bbN,$
        \\
        $n_{\mathrm{in}} \in \bbN,
        \mathrm{ws\_size} \in \bbN,
        \epsilon > 0
        $}
    \caption{\texttt{skglm} (proposed)}\label{alg:skglm}
    \For{$t = 1, \ldots, n_{\mathrm{out}}$}{

            $
            \mathrm{score}
            =
            \big(\dist \left(
                - \nabla_j f(\beta), \partial g_j(\beta_j) \right) \big)_{j \in [p]}
            $

        $\mathrm{ws\_size}
        =
        \max (
            \mathrm{ws\_size},
            2 \times |\gsupp (\beta)| )$

        \tcp{$\mathrm{ws\_size}$ features with largest scores}

        $\mathrm{ws} = \mathrm{arg\_topK}(\mathrm{score}, \mathrm{K=ws\_size})$

        \lIf{$\max_{j \in [p]} \dist \left(
                    - \nabla_j f(\beta), \partial g_j(\beta_j) \right) \leq \epsilon$}{
            stop
        }
        \lElse{
            \tcp{accelerated CD on working set}

            $\beta
            \leftarrow
            \texttt{inner\_solver}(X, \beta, \mathrm{ws},
            n_{\mathrm{in}}, \epsilon)$
        }
    }
  \Return{$\beta$}
\end{algorithm}
\end{minipage}
\hfill
\begin{minipage}[t]{0.5\linewidth}
    \begin{algorithm}[H]
        \SetKwInOut{Init}{init}
        \SetKwInOut{Input}{input}
        \setcounter{AlgoLine}{0}
        \Input{$
            X,
            \beta^{(0)} \in \bbR^p,
            \mathrm{ws} \subset [p],
            % n_{\mathrm{in}} \in \bbN$,
            n_{\mathrm{in}}$,
            $\epsilon$,
            % $\epsilon>0$,
            $M=5$
            }
        \caption{ \texttt{inner\_solver}}\label{alg:subproblem}

                \For{$k = 1, \ldots, n_{\mathrm{in}}$}{
                $\beta^{(k)} \leftarrow
                \texttt{CD}(X, \beta^{(k-1)}, X \beta, \mathrm{ws}) $
                % \hspace{-2em}\tcp*{$\bigo(n |\mathrm{ws}|)$}
                \hspace{-2em}\tcp*{Algo.~\eqref{alg:cd_epoch}}

                \If{$k \, \mathrm{mod} \, M = 0$}{
                    \tcp{
                        Algo.~\eqref{alg:extrapolation}, $\bigo(M^2 |\mathrm{ws}| + M^3)$}
                    $\beta_{\mathrm{ws}}^{\mathrm{extr}}
                    \leftarrow
                    \texttt{Anderson}(\beta_{\mathrm{ws}}^{(k - M)},
                    \dots,
                    \beta_{\mathrm{ws}}^{(M)})$

                    \tcp{
                        test objective $\bigo(n |\mathrm{ws}|)$}
                    \If{
                        $\Phi(
                            \beta_{\mathrm{ws}}^{\mathrm{extr}})
                        <
                        \Phi(
                            \beta_{\mathrm{ws}}^{(k)})
                        $
                    }{
                        $\beta_{\mathrm{ws}}^{(k)}
                        \leftarrow \beta_{\mathrm{ws}}^{\mathrm{extr}} ; X\beta \leftarrow X_{\mathrm{ws}}\beta_{\mathrm{ws}}^{\mathrm{extr}}$

                        % $X\beta \pluseq X_{\mathrm{ws}} (\beta_{\mathrm{ws}}^{\mathrm{extr}} - \beta_{\mathrm{ws}}^{(k)})$
                        }
                }
                \lIf{
                    $\max_{j \in \mathrm{ws}} \dist \left(
                        - \nabla_j f(\beta), \partial g_j(\beta_j) \right) \leq \epsilon $
                }{
                    stop
                }
            }
      \Return{$\beta^{(k)}$
      }
    \end{algorithm}
    \vspace{1mm}
\end{minipage}

Proof of \Cref{prop:conv_ws} can be found in \Cref{app:sub_conv_ws}.
The second key ingredient to our algorithm is to use state-of-the-art Anderson accelerated coordinate descent for non-convex problems.
In \Cref{sub:cd_analysis} we show that coordinate descent yields finite time support identification for a large class of non-convex problems (\Cref{prop:finite_identification}), which leads to %local linear convergence rate (\Cref{prop:dl_fixed_point}), and
acceleration (\Cref{prop:acc_lin_conv}).
As experiments demonstrate in \Cref{sec:experiments}, this rate allows our algorithm to surpass state-of-the-art solvers.
%
%%%%%%%%%%%%%%%%%%%%%%%%%%%%%%%%%%%%%%%%%%%%%%%%%%%%%%%%%%%%%%%%%%%%%%%%%%%%%%
\subsection{Anderson accelerated coordinate descent analysis for $\alpha$-semi-convex penalties}
%%%%%%%%%%%%%%%%%%%%%%%%%%%%%%%%%%%%%%%%%%%%%%%%%%%%%%%%%%%%%%%%%%%%%%%%%%%%%%
\label{sub:cd_analysis}
We now turn to our main technical contributions:
we show that \Cref{alg:subproblem} achieves finite time support identification (\Cref{prop:finite_identification}) of the generalized support (\Cref{def:generalized_supp}) for specific class of non-smooth non-convex penalties (\Cref{ass:alpha_semi_cvx}), which includes the MCP (\Cref{prop:mpc_alpha_semicvx}).
Based on \Cref{prop:finite_identification}, we are able to derive convergence rates for Anderson acceleration (\Cref{prop:acc_lin_conv}).

We study our inner solver (\Cref{alg:subproblem}); for convenience we still refer to $\beta$ and $X$ for their counterparts restricted to the working set.
% We show that coordinate descent achieves support identification (\Cref{prop:finite_identification}) and local linear convergence (\Cref{prop:dl_fixed_point}) for a class of non-convex functions.
The following assumptions are required.
\begin{assumption}[$\alpha$-semi-convex]\label{ass:alpha_semi_cvx}
    For all $j \in [p]$ $g_j/L_j$ is $\alpha$-semi-convex, \ie $g_j/L_j + \alpha \normin{\cdot}^2 / 2$ is convex, with $\alpha < 1$.
\end{assumption}
Note that in statistics, the admissible value range of hyperparameters for  MCP and SCAD are datafit-dependent, (see \citealt[Sec. 2.1]{Breheny_Huang2011}, normalized columns and
$\gamma > 1 = 1 / \norm{X_{:j}} = 1 / L_j$
or \citealt[Eq. 4.2]{Soubies_BlancFeraud_Aubert2015})
and yields $\alpha$-semi-convexity for MCP and SCAD\footnote{However MCP and SCAD are not $\alpha$-semiconvex for all hyperparameter values.}.
\begin{proposition}[$\alpha$-semi-convexity of MCP]
    \label{prop:mpc_alpha_semicvx}
    Let
    % \begin{align}
        $\MCP_{\lambda, \gamma}(x)
        \eqdef
        \begin{cases}
            \lambda \abs{x} - \frac{x^2}{2 \gamma} \enspace,
            &\text{if}\: \abs{x} \leq \gamma \lambda \enspace, \\
                \frac{1}{2} \gamma \lambda^2 \enspace,
            &\text{if}\: \abs{x} > \gamma \lambda \enspace .
        \end{cases}$
    % \end{align}
    \\
    If $\gamma > 1 / L_j$,
    then $\MCP_{\lambda, \gamma} / L_j$ is $\alpha$-semi-convex with $\alpha=\tfrac{1}{2}(1 + \frac{1}{\gamma L_j})$ (\ie \Cref{ass:alpha_semi_cvx} holds).
\end{proposition}
%

% for the admissible range of hyperparameters, \Cref{ass:alpha_semi_cvx} is verified for the MCP and SCAD.
Note that \Cref{ass:alpha_semi_cvx} does not hold for the $\ell_q$-penalties ($0 < q < 1$), for which we propose an alternative in \Cref{app:lq}.
\begin{assumption}[Existence]\label{ass:non_empty}
    \Cref{pb:generic_pb} admits at least one critical point.
\end{assumption}
In \Cref{prop:finite_identification}, convergence of \Cref{alg:subproblem} toward a critical point $\hat \beta$ is assumed, and the following assumption is made on this critical point.
\begin{assumption}[Non degeneracy]\label{ass:non_degeneracy}
The considered critical point $\hat \beta \in \bbR^p$ is non-degenerated:
for all $j \notin \gsupp(\hat \beta)$,
\begin{equation}
    - \nabla f_j (\hat \beta) \in \, \mathrm{interior}(\partial g_j(\hat \beta_j)).
\end{equation}
% \begin{equation}
%     - \nabla f (\hat \beta) \in \mathrm{ri} \, \partial g(\hat \beta)
%     \enspace .
% \end{equation}
\end{assumption}
\Cref{ass:non_degeneracy} is a generalization of qualification constraints \citep[Sec. 1]{Hare_Lewis2007}, and is usual in the machine learning literature \citep{Zhao_Yu2006,Bach08,Vaiter_Peyre_Fadili2015}.
For the $\ell_1$-norm, if the entries of the design matrix $X$ are drawn from an i.i.d normal distribution, then \Cref{ass:non_degeneracy} holds with high probability \citep{Candes_Tao2005,Rudelson_Vershynin2008}.

Equipped with the previous assumptions we show that coordinate descent achieves model identification for this class of non-convex problems.
\begin{proposition}[Model identification of CD]
    \label{prop:finite_identification}
   Suppose
    \begin{enumerate}[itemsep=-2pt,topsep=0pt]
        \item \Cref{ass:f,ass:g,ass:non_empty,ass:alpha_semi_cvx} hold.
        \item The sequence $(\beta^{(k)})_{k\geq 0}$ generated by coordinate descent (\Cref{alg:subproblem}  without extrapolation) converges toward a critical point $\hat \beta$.
        \item \Cref{ass:non_degeneracy} holds for $\hat \beta$.
    \end{enumerate}
    Then, \Cref{alg:subproblem} (without extrapolation)  identifies the model in finitely many iterations: there exists $K>0$ such that for all
    $k\geq K$,
    $\beta_{\cS^c}^{(k)} = \hat \beta_{\cS^c}$.
\end{proposition}
% \begin{remark}[Comments on \Cref{prop:finite_identification} and its improvement over existing work]
    In other words, for $k$ large enough,
    $\beta^{(k)}$ shares the generalized support of $\hat \beta$.
    The identification property was proved for a proximal gradient descent algorithm in the non-convex case \citep{Liang_Faidli_Peyre_2016} under the assumption that the non-smooth function $g$ is partly smooth \citep{Lewis_2002}.
    For ourselves, \Cref{prop:finite_identification} not rely on the partly smooth assumption to ensure identification property.
    Authors are not aware of previous identification results for coordinate descent in the non-convex case.

In addition, if $f$ and $g$ are locally regular on the generalized support at the considered critical point, our algorithm enjoys local acceleration when combined with Anderson extrapolation (\Cref{prop:acc_lin_conv}).
\begin{assumption}[Locally $\mathcal{C}^3$]\label{ass:locally_c2}
    For all $j \in \mathcal{S} \eqdef \gsupp(\hat \beta)$, $g_j$ is locally $\mathcal{C}^3$ around $\hat \beta_j$, and $f$ is locally $\mathcal{C}^3$ around $\hat \beta$.
\end{assumption}
\Cref{ass:locally_c2} on the function $f$ is mild and holds for usual machine learning datafitting terms.
\Cref{ass:locally_c2} on the functions $g_j$, $j \in \cS$, is stronger: for instance, for the MCP, it implies $\hat \beta_j \neq \gamma \lambda$ for all $j \in \cS$.
However this assumption is standard in the literature, see \citealt[Sec. 3.3]{Liang_Faidli_Peyre_2016}
\begin{assumption} (Local strong convexity)\label{ass:local_str_cvx}
    % For a solution $\hat \beta\in \argmin_{\beta\in\bbR^p} \Phi(\beta)$,
    The Hessian of $f$ at the considered critical point $\hat \beta \in \bbR^p$, restricted to its generalized support $\cS$, is positive definite, \ie
        $\nabla^{2}_{\cS, \cS}f(\hat \beta)
        +
        \nabla^{2}_{\cS, \cS} g(\hat \beta)
        \succ 0$.
\end{assumption}
\Cref{ass:local_str_cvx} requires local strong convexity restricted to the generalized support $\cS$, which is standard in the MCP / SCAD literature (\citealt[Section 4.1]{Breheny_Huang2011}) and is usual to derive local linear rates of convergence \citep[Section 3.3]{Liang_Faidli_Peyre_2016}.
For instance, for the Lasso, if the entries of the design matrix $X$ are drawn from a continuous distribution, then \Cref{ass:local_str_cvx} holds with probability one \citep[Lemma 4]{Tibshirani2013}.
% The local strong convexity hypothesis (\Cref{ass:local_str_cvx}) will ensure local linear convergence rates (\Cref{prop:dl_fixed_point}).
%
%
\begin{proposition} \label{prop:acc_lin_conv}
    \sloppy
    Consider a critical point $\hat \beta$ and suppose
    \begin{enumerate}[itemsep=-2pt,topsep=0pt]
        \item \Cref{ass:f,ass:g,ass:non_empty} hold.
        \item The functions $f$ and $g_j$, $j \in [p]$ are piecewise quadratic (which is the case for the MCP regression).
        \item The sequence $(\beta^{(k)})_{k\geq 0}$ generated by Anderson accelerated coordinate descent with updates from $1$ to $p$ and $p$ to $1$ (\Cref{alg:subproblem} with extrapolation) converges to a critical point $\hat \beta$.
        \item \Cref{ass:non_degeneracy,ass:locally_c2,ass:local_str_cvx} hold for $\hat \beta$.
    \end{enumerate}
        Then there exists
        $K \in \mathbb{N}$, and a $\mathcal{C}^1$ function
        $\psi: \mathbb{R}^{|\cS|} \to \mathbb{R}^{|\cS|}$
        such that, for all $k \in \mathbb{N}, k \geq K$:
        \begin{align}
            \beta_j^{(k)}
            &=
            \hat \beta_j
            \text{ , for all $j \in \cS^c$, }
        \end{align}
        Let
        $T \eqdef  \cJ \psi (\hat \beta)$,
        $H \eqdef \nabla_{\cS, \cS}^{2} f(\hat \beta) + \nabla_{\cS, \cS}^{2} g(\hat \beta)$,
        $\zeta \eqdef (1 - \sqrt{1 - \rho(T)}) / (1 + \sqrt{1- \rho(T)})$ and
        $B \eqdef ( T - \Id)^\top ( T - \Id)$.
        Then $\rho(T) < 1$ and the iterates of Anderson extrapolation enjoy local accelerated convergence rate:
        \begin{align}
            &\normin{\beta_{\cS}^{(k-K)} - \hat \beta_{\cS}}_B
            \leq
            % \nonumber
            % \\
            \Big(
                \sqrt{\kappa(H)}
                \tfrac{2\zeta^{M-1}}{1 + \zeta^{2(M-1)}} \Big)^{(k-K) / M}
            \normin{\beta_{\cS}^{(K)} - \hat \beta_{\cS}}_B
            \enspace .
        \end{align}
\end{proposition}
The proof can be found in \Cref{app:local_acc}.

\textbf{Related work.}
Most Anderson acceleration convergence results are shown for quadratic objectives for specific algorithms: gradient descent \citep{Golub_Varga1961,Anderson65}, ADMM \citep{Poon_Liang2019}, coordinate descent \citep{bertrand2020implicit}.
Outside of the quadratic case, convergence results are usually significantly weaker \citep{Scieur_dAspremont_Bach2016,Sidi17,Brezinski2018,Mai_Johansson_19,Ouyang_Peng_Yao_Zhang_Deng2020}.
Regarding the smooth non-convex case, \citet{Wei_Bao_Liu2021} proposed a stochastic Anderson acceleration and proved convergence towards a critical point.
\Cref{prop:acc_lin_conv} generalizes \citet[Prop 2.1]{Scieur_dAspremont_Bach2020} and \citet[Prop. 4]{Bertrand_Massias2020} to the proximal convex and $\alpha$-semi-convex cases.
To our knowledge this is one of the first quantitative results for Anderson acceleration in a non-convex setting.
%
%
%%%%%%%%%%%%%%%%%%%%%%%%%%%%%%%%%%%%%%%%%%%%%%%%%%%%%%%%%%%%%%%%%%%%%%%%%%%%
\subsection{Comparison with existing work}
%%%%%%%%%%%%%%%%%%%%%%%%%%%%%%%%%%%%%%%%%%%%%%%%%%%%%%%%%%%%%%%%%%%%%%%%%%%%
%
\begin{table*}[tp]
    \center
    \caption{Most popular packages for sparse generalized linear models.}
    \begin{tabular}{c|cccl}
        % \hline
         Name
         & Acceleration
         & Huge scale
         & Nncvx
         & Modular\\
        \hline
        \hline
        \texttt{glmnet} \citep{Friedman_Hastie_Tibshirani10}
            & \xmark & \xmark & \xmark & \xmark
            \xspace  \ \ (\textcolor{red}{Fortran}) \\
        \texttt{scikit-learn} \citep{Pedregosa_etal11}
            & \xmark & \xmark & \xmark & \xmark
            \xspace \ \ (\textcolor{red}{Cython}) \\
        \texttt{lightning} \citep{Blondel_Pedregosa2016}
            & \xmark & \xmark & \xmark & \cmark
            \xspace \ \ (\textcolor{red}{Cython}) \\
        \texttt{celer} \citep{Massias_Gramfort_Salmon2018}
            & \cmark & \cmark & \xmark & \xmark
            \xspace \ \ (\textcolor{red}{Cython}) \\
        \texttt{picasso} \citep{picasso}
            &\xmark & \xmark & \cmark & \xmark
            \xspace \ \  (\textcolor{red}{\texttt{C++}}) \\
        \texttt{pyGLMnet} \citep{pyGMLnet}
            & \xmark & \xmark \xmark & \xmark & \cmark
            \xspace \ \ (\textcolor{olive}{Python}) \\
        \texttt{fireworks} \citep{Rakotomamonjy_Flamary_Gasso_Salmon20}
            & \xmark & \cmark & \cmark & N.A.
            \xspace \ \ (\textcolor{olive}{Python}) \\
        \texttt{skglm} (ours)
            &\cmark & \cmark & \cmark \cmark & \cmark
            \xspace \ \  (\textcolor{olive}{Python}) \\
    \end{tabular}
    \label{table:summary_packages}
\end{table*}

In this section we compare our contribution to existing algorithms and implementations, which are summarized in \Cref{table:summary_packages}.
\emph{Huge scale} refers to the fact that the algorithm can run on problems with millions of variables. %and handle large sparse design matrices.
\emph{Non-convex} tells if the algorithm handles non-convex penalties.
\emph{Modular} indicates that it is easy to add a new model, through a different datafitting term or penalty.

The packages \texttt{glmnet} \citep{Friedman_Hastie_Tibshirani10}, \texttt{scikit-learn} \citep{Pedregosa_etal11} and \texttt{lightning} \citep{Blondel_Pedregosa2016} implement coordinate descent (cyclic or random).
They rely on compiled code such as Fortran or Cython, making it very difficult to implement new models\footnote{%
\url{https://github.com/scikit-learn/scikit-learn/pull/10745} (4 years old)}
or faster algorithms like working set\footnote{
    \url{https://github.com/scikit-learn/scikit-learn/pull/7853} (5 years old)}.
They do not handle non-convex penalties.

More recent algorithms such as \texttt{blitz} \citep{Johnson_Guestrin15}, \texttt{celer} \citep{Massias_Gramfort_Salmon2018}, \texttt{picasso} \citep{picasso} or \texttt{fireworks} \citep{Rakotomamonjy_Flamary_Gasso_Salmon20} use working set strategies.
\texttt{celer} and \texttt{blitz} are state-of-the-art algorithms for the Lasso, but their score to prioritize features relies on duality.
\texttt{fireworks} extends \texttt{blitz} to some non-convex penalties (writing as difference of convex functions), with
    $\mathrm{score}_{j}^{\mathrm{fireworks}}
    =
    \dist(-\nabla_j f (\beta), \partial g_j (0))$.
Yet this rule does not consider the subdifferential of $g$ at the current point, but at 0, which is a coarse information.
Finally, \texttt{fireworks}, building upon the seminal non convex working set solver of \citet{Boisbunon_Flamary_Rakotomamonjy}, does not provide accelerated convergence rates and does not come with a public implementation.
\texttt{picasso} \citep{picasso} lacks modularity (penalties are hardcoded), and the solver is not suited for huge scale (it does not support large sparse matrices).
\citet{Deng_Lan2019} proposed an algorithm based on inertially accelerated coordinate descent, which fails to provide practical speedups according to \citet{Bertrand_Massias2020}.

Contrary to these algorithms, ours is generic and relies only on the knowledge of $\nabla f$ and $\prox_g$.
For any new penalty, this information can be written in a few lines of Python code, compiled with numba \citep{numba} for speed efficiency.
We therefore improve state-of-the-art algorithms in the convex case, and
generalize to virtually any datafit and penalty, even nonconvex.

\section{Experiments}\label{sec:experiments}
%%%%%%%%%%%%%%%%%%%%%%%%%%%%%%%%%%%%%%%%%%%%%
%
\begin{figure*}[tb]
    \centering
    \includegraphics[width=0.8\linewidth]{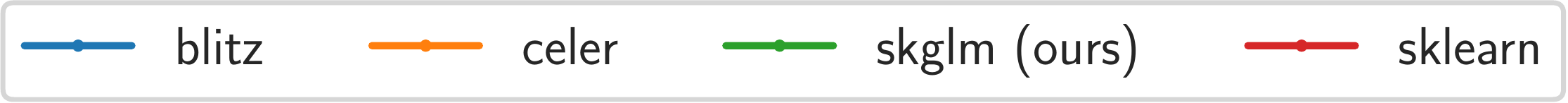}
    \includegraphics[width=0.9\linewidth]{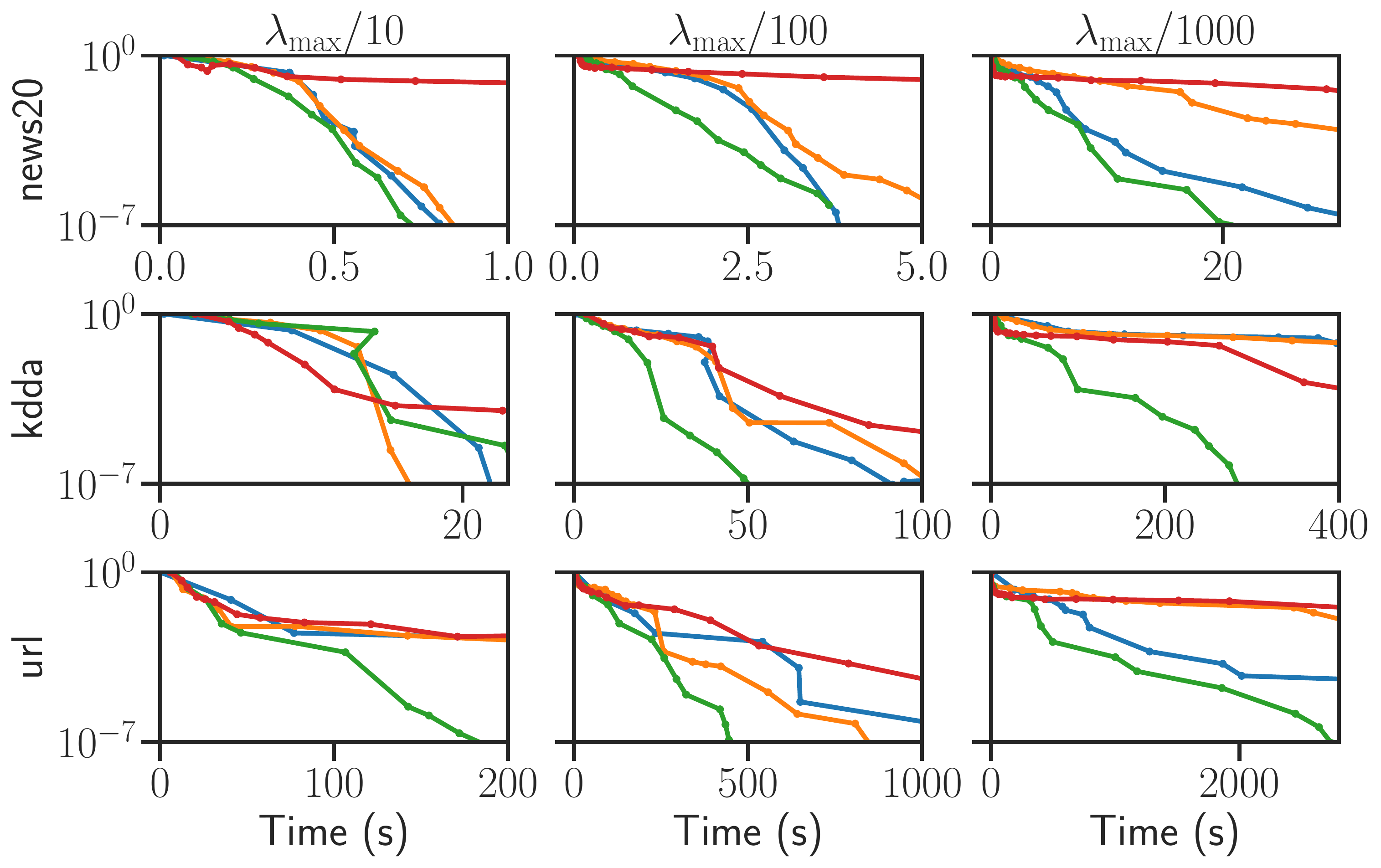}
    \caption{\textbf{Lasso, duality gap.}
    Normalized duality gap as a function of time for the Lasso on multiple datasets, for multiple values of $\lambda$.
    }
    \label{fig:duality_gap_lasso}
\end{figure*}
%%%%%%%%%%%%%%%%%%%%%%%%%%%%%%%%%%%%%%%%%%%%%%%%%%%%%%%%%%%%%%%
%
\sloppy
Our package relying on numpy and numba \citep{numba,harris2020array} is attached in the supplementary material.
An open source, fully tested and documented version of the code can be found at \url{https://github.com/scikit-learn-contrib/skglm}.
We use datasets from \texttt{libsvm}\footnote{\url{https://www.csie.ntu.edu.tw/~cjlin/libsvmtools/datasets/}} (\citealt{Fan_Chang_Hsieh_Wang_Lin08}, see \cref{table:summary_data}).

We compare multiple algorithms to solve popular Machine Learning and inverse problems: Lasso, Elastic net, multitask sparse regression, MCP regression.
The compared algorithms are the following:
\begin{itemize}[itemsep=0pt,topsep=0pt]
    \item \texttt{scikit-learn} \citep{Pedregosa_etal11}, which implements coordinate descent in Cython, %with cyclic index selection in Cython.
    \item \texttt{celer} \citep{Massias_Vaiter_Salmon_Gramfort2019}, which combines working sets, screening rules, coordinate descent, and Anderson acceleration in the dual, in Cython,
    \item \texttt{blitz} \citep{Johnson_Guestrin15}, which combines working sets with prox-Newton iterations \citep{Lee_Sun_Saunders2012} in C++,
    % \item \texttt{lightning} \citep{Blondel_Pedregosa2016}, which implements coordinate descent and shrinking heuristics in Cython,
    % The unsafe shrinking heuristics sometimes lead to non-convergence, in which case we disabled it,
    \item coordinate descent (CD, \citealt{tseng2009coordinate}),
    \item \texttt{skglm} (\Cref{alg:skglm}, ours), using $M=5$ iterates for the Anderson extrapolation.
\end{itemize}

\textbf{Other solvers.}
Experiment per experiment, there exist niche solvers (such as aggressive Gap Safe Rules, \citealt{Ndiaye_Fercoq_Salmon20}).
Since our goal is a \emph{general purpose} algorithm able to deal with many models, we do not include them in the comparison.
In addition, we focus on solving a single instance of \Cref{pb:generic_pb}, rather than a regularization path (\ie a sequence of problems for multiple regularization strengths).
As \texttt{glmnet} is designed to compute regularization paths, we could not include it in the comparison.
The reader can refer to \citet[Fig. 4]{Johnson_Guestrin15} or \Cref{fig:duality_gap_enet_suppl} in \Cref{app:expe_suppl} for comparisons on single optimization problems with \texttt{glmnet};
\texttt{glmnet} and additional algorithms are discussed in \Cref{app:expe_suppl}.

\textbf{How to do a fair comparison between solvers?}
To plot the convergence curves, we use the \texttt{benchopt}\footnote{\url{https://github.com/benchopt/benchopt}} benchmarking package \citep{benchopt}.
In order to automate and reproduce optimization benchmarks it treats solvers as black boxes.
It launches them several times with increasing maximum number of iterations, and stores the resulting objective values and times to reach it.
As each point on a solver curve is obtained in a different run, the curves are not monotonic, and there may be several points corresponding to the same time.
This merely reflects the variability in solvers running time across runs; we refer to \Cref{fig:non_monotonous_curves} in \Cref{app:variability} for the inevitability of this phenomenon with black box solvers.
%
%%%%%%%%%%%%%%%%%%%%%%%%%%%%%%%%%%%%%%%%%%%%%%%%%%%%%%%%%%%%%%%
\subsection{Convex problems}
%%%%%%%%%%%%%%%%%%%%%%%%%%%%%%%%%%%%%%%%%%%%%%%%%%%%%%%%%%%%%%%
\paragraph{Lasso.}
%%%%%%%%%%%%%%%%%%%%%%%%%%%%%%%%%%%%%%%%%%%%%%%%%%%%%%%%%%%%%%%
In \Cref{fig:duality_gap_lasso} we compare solvers for the Lasso:
% \begin{problem}\label{pb:lasso}
%     \argmin_{\beta \in \bbR^p}
%     \frac{1}{2n} \normin{y - X \beta}^2
%     + \lambda \normin{\beta}_1
%     \enspace .
% \end{problem}
($f = \frac{1}{2n}\norm{y - X\cdot}^2$, $g_j = \lambda \abs{\cdot}$).
We parametrize $\lambda$ as a fraction of $\lambda_{\text{max}} = \normin{X^\top y}_\infty / n$, smallest regularization strength for which $\hat \beta = 0$.
For large scale datasets (\emph{rcv1}, \emph{news20}),
\texttt{skglm} yields performances better or similar to the state-of-the-art algorithms \texttt{blitz} and \texttt{celer}.
For huge scale datasets (\emph{kdda} and \emph{url}), \texttt{skglm} yields significant speedups over them.
The improvement over the popular \texttt{scikit-learn} can be of two orders of magnitude.
Thus, \emph{while dealing with many more models, our algorithm still yields state-of-the-art speed for basic ones}.
%%%%%%%%%%%%%%%%%%%%%%%%%%%%%%%%%%%%%%%%%%%%%%%%%%%%%%%%%%%%%%
%

%
\begin{figure*}[tb]
    \centering
        \includegraphics[width=0.7\linewidth]{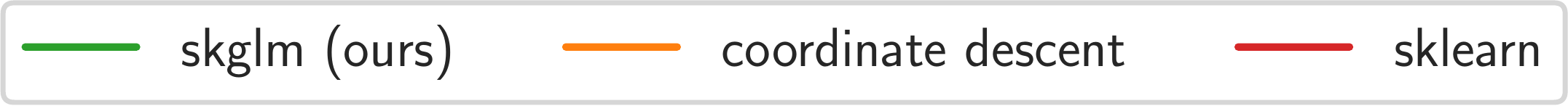}
        \includegraphics[width=0.9\linewidth]{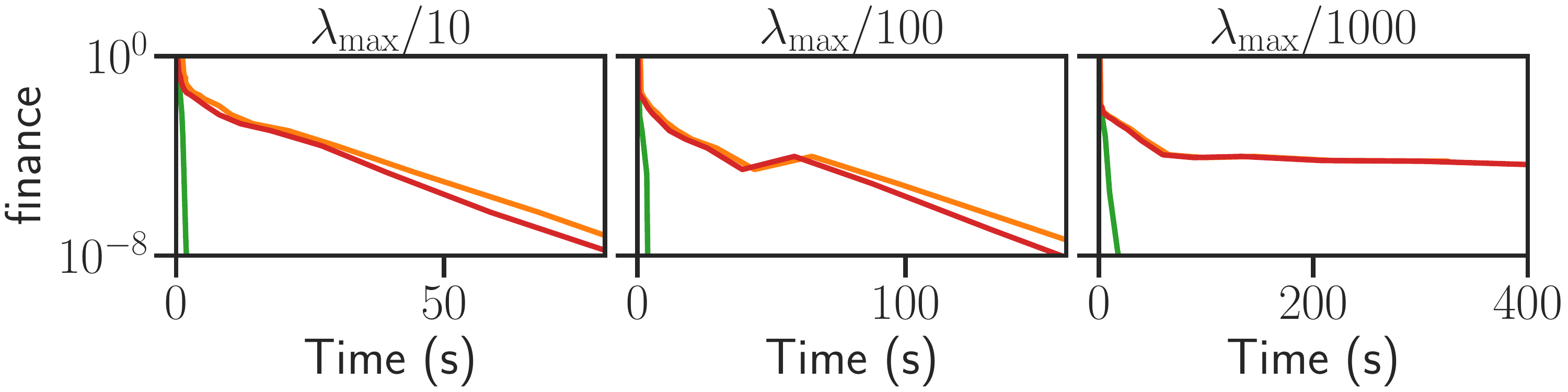}
    \caption{\textbf{Elastic net, duality gap.}
    Normalized duality gap as a function of time for the elastic net for multiple values of $\lambda$, $\rho=0.5$.}
    \label{fig:duality_gap_enet}
\end{figure*}
\begin{figure}[t]
    \centering
    \subfigure[$\ell_{2, 1}$.] {
        \label{fig:meg_l21}
        \includegraphics[scale=0.24]{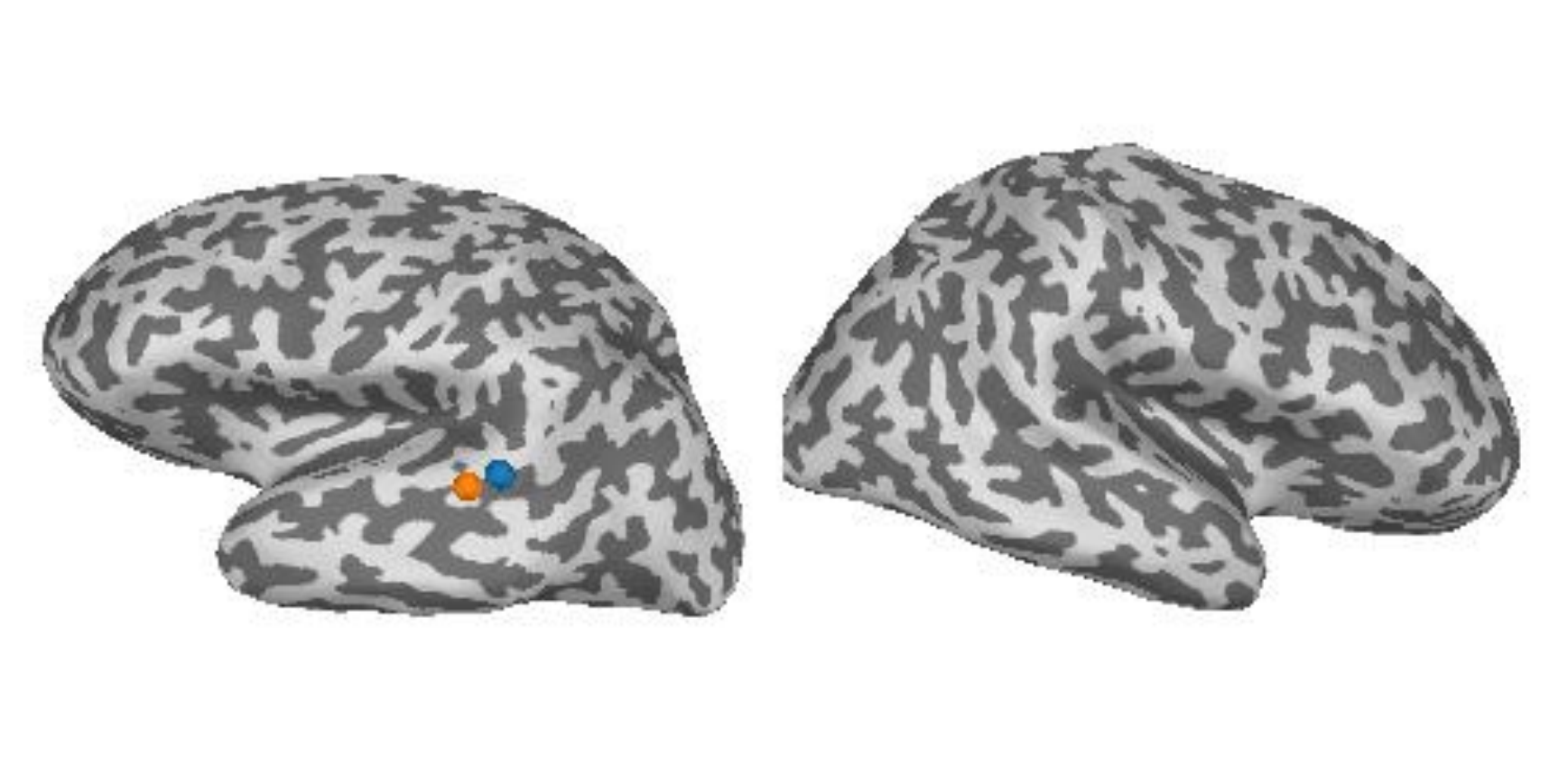}
    }
    \subfigure[$\ell_{2, 0.5}$.] {
        \label{fig:meg_l205}
        \includegraphics[scale=0.24]{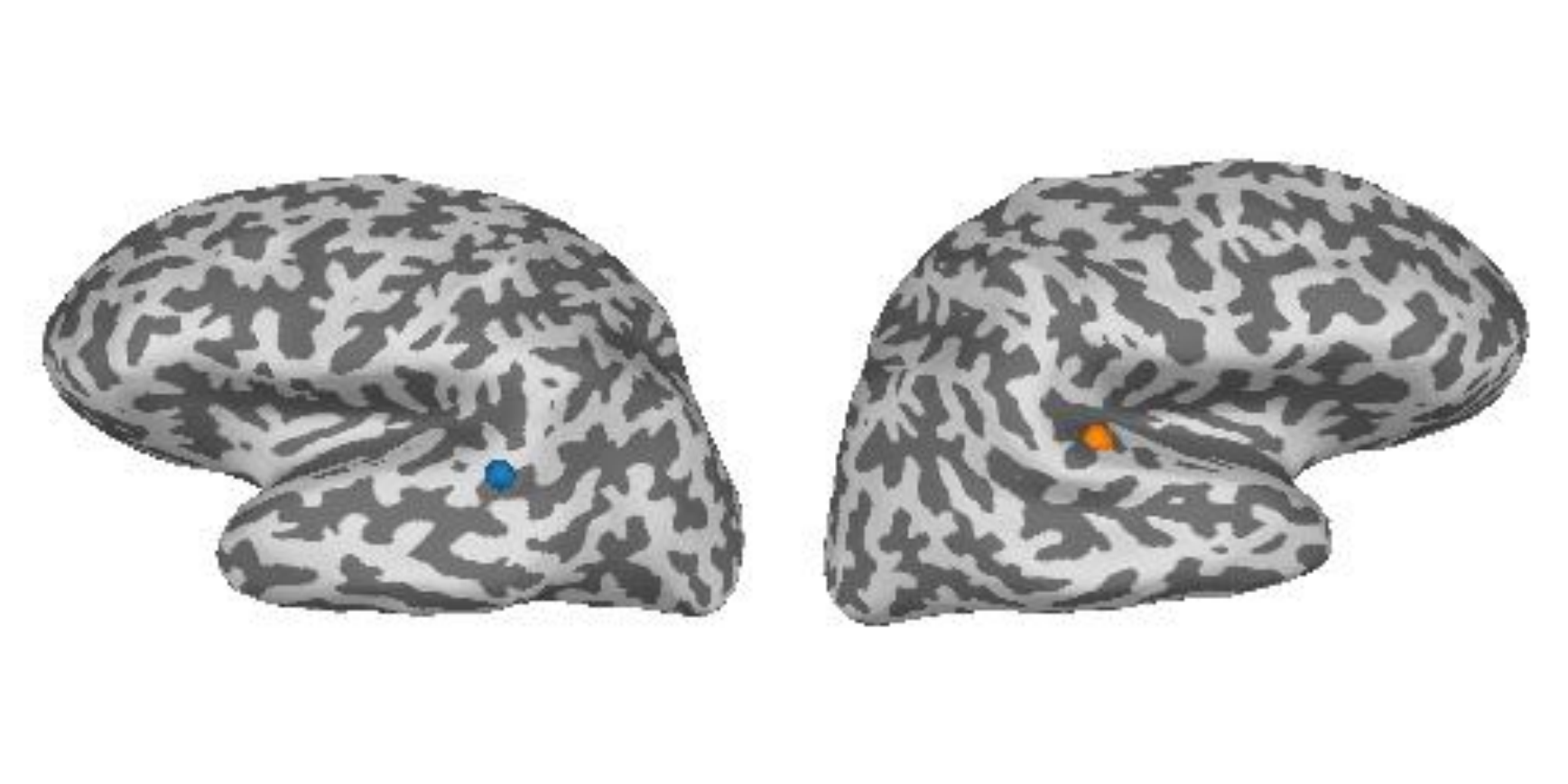}
    }
    \subfigure[Block MCP.] {
        \label{fig:meg_block_mcp}
        \includegraphics[scale=0.24]{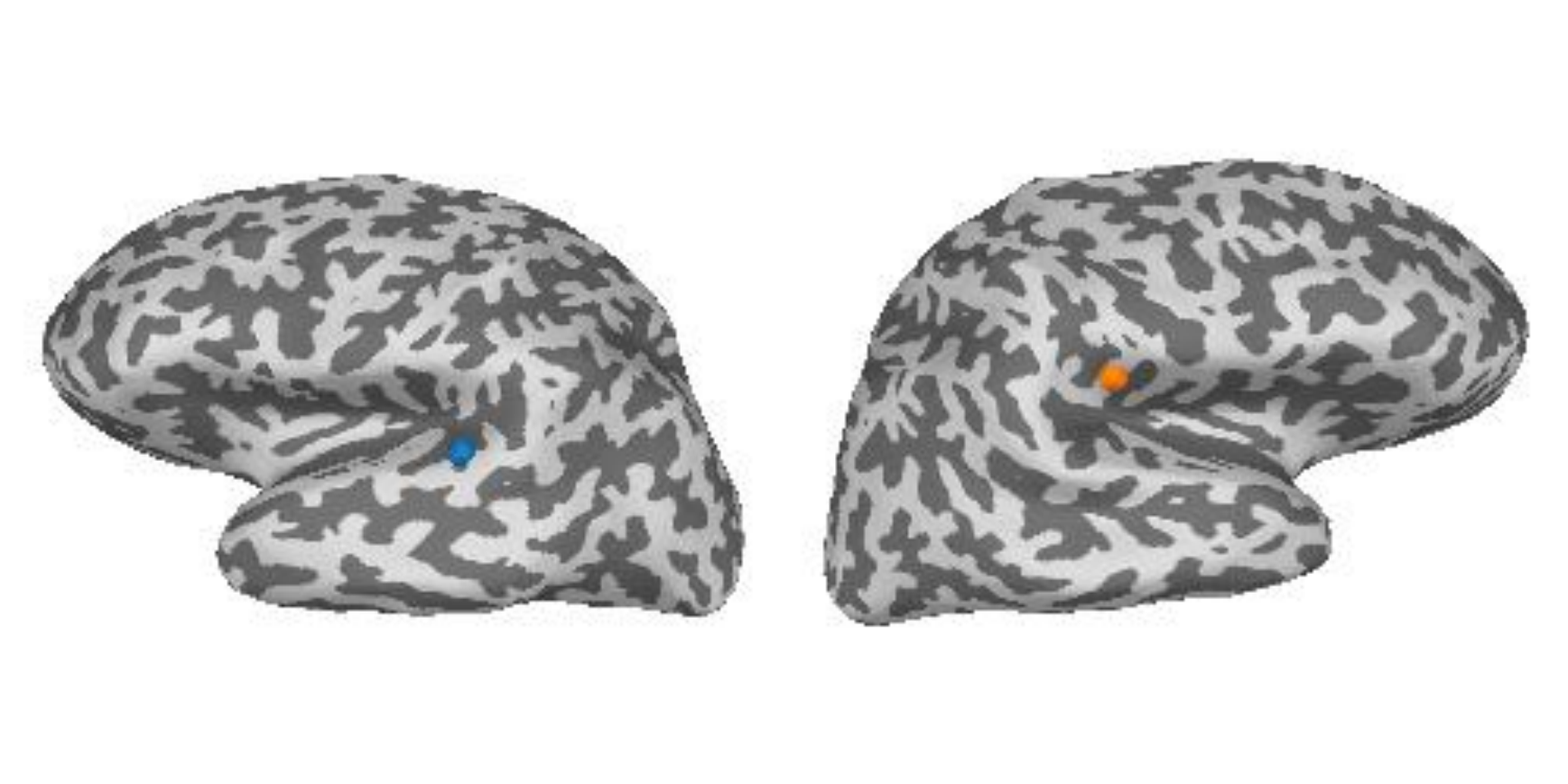}
    }
    \caption[placeholder]{
        \textbf{Real data, brain source locations recovered by convex and non-convex penalties after a right auditory stimulation.}
        \ref{fig:meg_l21} shows that a convex penalty fails at identifying one source in each hemisphere, while \ref{fig:meg_l205} and \ref{fig:meg_block_mcp}
        demonstrates the capability of non-convex penalties to recover the correct solution.
    }
    \label{fig:expes_meg}
\end{figure}
\begin{figure*}[tb]
    \centering
        \includegraphics[width=0.8\linewidth]{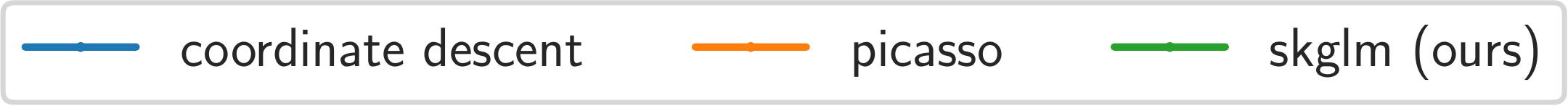}\\
        \includegraphics[width=1\linewidth]{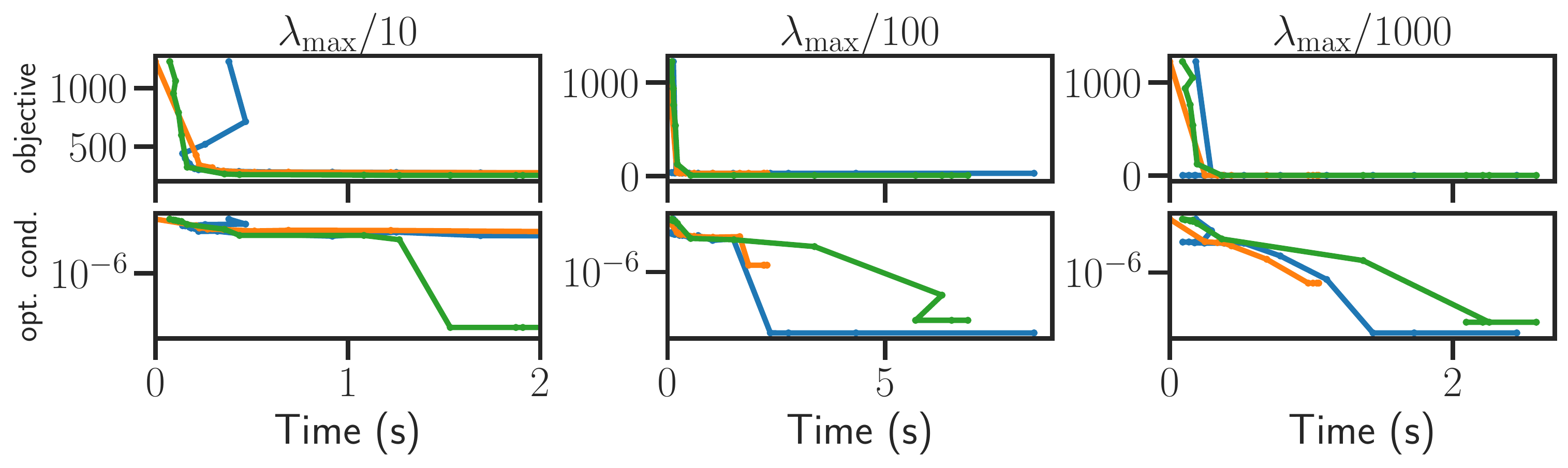}
        \includegraphics[width=0.9\linewidth]{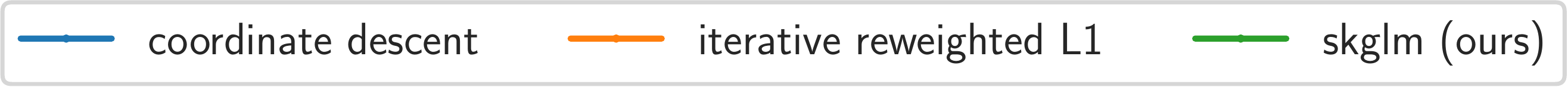}
        \includegraphics[width=1\linewidth]{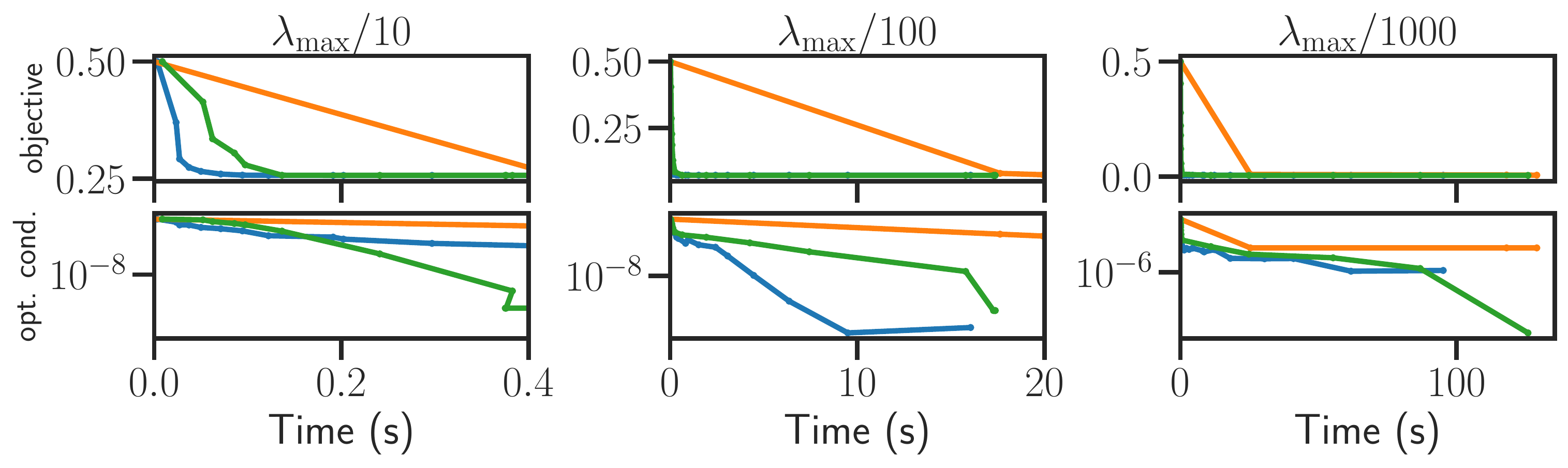}
    \caption{\textbf{MCP, objective value and violation of first order condition.}
    Objective value and violation of optimality condition of the iterates, $\mathrm{dist} (- \nabla f(\beta^{(k)}), \partial g(\beta^{(k)}) )$, as a function of time for the MCP for multiple values of $\lambda$ ($\gamma=3$) on a simulated dense dataset (top) and the rcv1 dataset (normalized columns).}
    \label{fig:mcp}
\end{figure*}

%%%%%%%%%%%%%%%%%%%%%%%%%%%%%%%%%%%%%%%%%%%%%%%%%%%%%%%%%%%%%%%
\paragraph{Elastic net.}
%%%%%%%%%%%%%%%%%%%%%%%%%%%%%%%%%%%%%%%%%%%%%%%%%%%%%%%%%%%%%%%
Our approach easily generalizes to other problems, such as the elastic net ($f = \frac{1}{2n}\norm{y - X\cdot}^2$, $g_j = \lambda (\rho \abs{\cdot} + \frac{1 - \rho}{2} (\cdot)^2)$).
% \begin{equation}
%     \argmin_{\beta \in \bbR^p}
%     \frac{1}{2n} \normin{y - X \beta}^2
%     + \lambda
%         \rho\normin{\beta}_1
%         + \lambda \tfrac{(1 - \rho)}{2} \normin{\beta}_2^2
%     \enspace .
% \end{equation}
\Cref{fig:duality_gap_enet} shows the duality gap as a function of time for \texttt{skglm} (ours), \texttt{sklearn}, and our numba implementation of coordinate descent.
The proposed algorithm is orders of magnitude faster than \texttt{scikit-learn} and vanilla coordinate descent, in particular for large datasets and low regularization parameter values (\emph{finance}, $\lambda_{\max} / 1000$).
Note that \texttt{blitz} does not implement a solver for the elastic net.
Many Lasso solvers would easily handle the elastic net, but relying on Cython/C++ code makes the implementation time-consuming.
By contrast, it takes $40$ lines of code to define an $\ell_1 + \ell_2$-squared penalty with our implementation.
%%%%%%%%%%%%%%%%%%%%%%%%%%%%%%%%%%%%%%%%%%%%%%%%%%%%%%%%%%%%%%%
%
% \paragraph{SVM with hinge loss}
An additional experiment on the dual of SVM with hinge loss is in \Cref{app:expe_svm}.
%%%%%%%%%%%%%%%%%%%%%%%%%%%%%%%%%%%%%%%%%%%%%%%%%%%%%%%%%%%%%%%
\subsection{Non-convex problems}
%%%%%%%%%%%%%%%%%%%%%%%%%%%%%%%%%%%%%%%%%%%%%%%%%%%%%%%%%%%%%%%
In this subsection we propose a comparison on two non convex problems. % the MCP regression problem and an application to neuroscience.
\paragraph{MCP regression.}
MCP regression is \Cref{pb:generic_pb} with $f = \frac{1}{2n} \norm{y - X \cdot}^2$, $g_j = \MCP_{\lambda,\gamma}$ for $\gamma > 1 $.
As usual for this problem, we scale the columns of $X$ to have norm $\sqrt{n}$.
On \Cref{fig:mcp}, we compare our algorithm to \texttt{picasso} on a dense dataset ($n=1000, p=5000$); as this package does not support large sparse design matrices, for the \emph{rcv1} dataset we use an iterative reweighted L1 algorithm \citep{Candes_Wakin_Boyd2008}.
Since the derivative of the MCP vanishes for values bigger than $\lambda \gamma$, this approach requires solving weighted Lassos with some 0 weights.
Up to our knowledge, our algorithm is the only efficient one with such a property.
Our algorithm handles problems of large size, converges to a critical point, and, due to its progressive inclusion of features, is able to reach a sparser critical point than it competitors.

%%%%%%%%%%%%%%%%%%%%%%%%%%%%%%%%%%%%%%%%%%%%%%%%%%%%%%%%%%%%%%%
\paragraph{Application to neuroscience}
%%%%%%%%%%%%%%%%%%%%%%%%%%%%%%%%%%%%%%%%%%%%%%%%%%%%%%%%%%%%%%%

To demonstrate the usefulness of our algorithm for practitioners, we apply it to the magneto-/electroencephalographic (M/EEG) inverse problem.
It consists in reconstructing the spatial cortical current density at the origin of M/EEG measurements made at the surface of the scalp.
Non-convex penalties \citep{Strohmeier_Gramfort_Haueisen2015} exhibit several advantages over convex ones \citep{Gramfort_Strohmeier_Haueisen_Hamalainen_Kowalski2013}: they yield sparser physiologically-plausible solutions and mitigate the $\ell_1$ amplitude bias.
Here the setting is multitask: $Y \in \bbR^{n \times T}$ and thus we use block penalties (details in \Cref{app:prox_multitask}).
We use real data from the \texttt{mne} software \citep{mne}; the experiment is a right auditory stimulation, with two expected neural sources to recover in each auditory cortex.
In \Cref{fig:expes_meg}, while the $\ell_{2, 1}$ penalty fails at localizing one source in each hemisphere, the non-convex penalties recover the correct locations.
This emphasizes on the critical need for fast solvers for non-convex sparse penalties as well as our algorithm's ability to handle the latter.
In this work we focused on optimization-based estimators to solve the inverse problem, note that one could have resort to other techniques, such as Bayesian techniques \citep{Ghosh2017,Fang2020}.

%%%%%%%%%%%%%%%%%%%%%%%%%%%%%%%%%%%%%%%%%%%%%%%%%%%%%%%%%%%%%%%
\textbf{Ablation study.}
%%%%%%%%%%%%%%%%%%%%%%%%%%%%%%%%%
To evaluate the influence of the two components of \Cref{alg:skglm}, an ablation study (\Cref{fig:abla_study}) is performed.
Four algorithms are compared: with/without working sets and with/without Anderson acceleration.
% The algorithm without working sets and acceleration is the usual coordinate descent algorithm, while working sets combined Anderson accelerated coordinate descent corresponds to the proposed \Cref{alg:skglm}.
\Cref{fig:abla_study} represents the duality gap of the Lasso as a function of time for multiple datasets and values of the regularization parameters $\lambda$ (parametrized as a fraction of $\lambda_{\text{max}}$).
First, \Cref{fig:abla_study} shows that working sets always bring significant speedups.
Then, when combined with working set, Anderson acceleration bring significant speed-ups, especially for hard problems with low regularization parameters.
An interesting observation is that on large scale datasets (\emph{news20} and \emph{finance}) and for low regularization parameters ($\lambda_{\max} / 100$ and $\lambda_{\max} / 1000$) Anderson acceleration \emph{without} working set does not bring acceleration.
This highlights the importance of combining Anderson acceleration with working sets.
%

%%%%%%%%%%%%%%%%%%%%%%%%%%%%%%%%%%%%%%%%%%%%%%%%%%%%%%%%%%%%%%%
\begin{figure*}[tb]
    \centering
        \centering
        \includegraphics[width=0.7\linewidth]{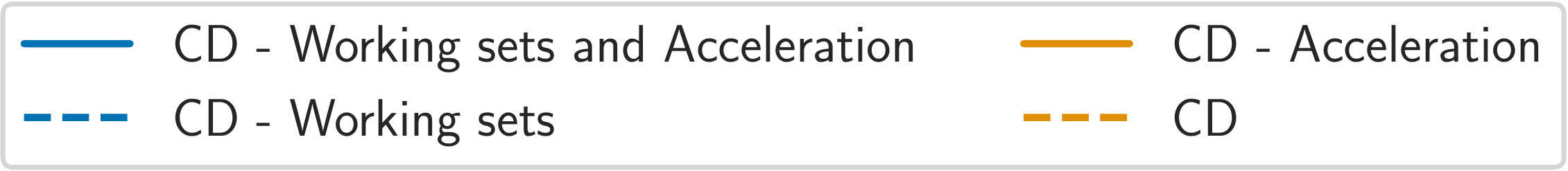}
        \includegraphics[width=0.9\linewidth]{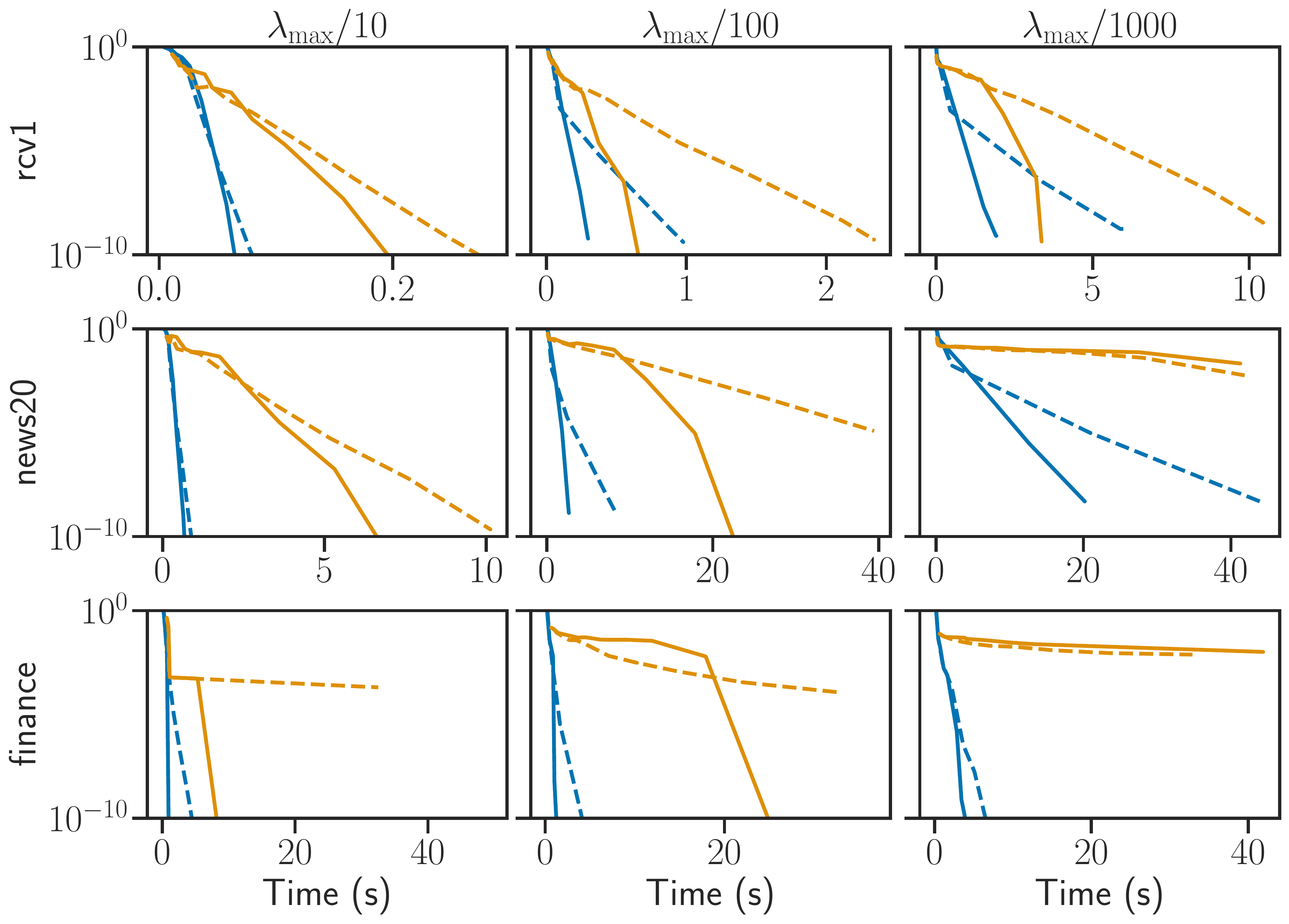}
    \caption{\textbf{Lasso, duality gap.}
    Normalized duality gap as a function of time for the Lasso. %, for multiple values of $\lambda$.
    }
    \label{fig:abla_study}
\end{figure*}
%%%%%%%%%%%%%%%%%%%%%%%%%%%%%%%%%%%%%%%%%%%%%%%%%%%%%%%%%%%%%%%

\paragraph{Conclusion and broader impact.}
In this paper, we have proposed an accelerated versatile algorithm for a specific class of non-smooth non-convex problems.
Based on working sets, coordinate descent and Anderson acceleration, we have improved state of the art on convex problems, and handled previously out-of-reach problems.
Thorough experiments demonstrated the speed and interest of our approach.
% As further directions of research, we plan on investigating the non-convexbiases of our algorithm, i.e. the characterization of its limit.
% \textbf{Limitations.}
A limitation of this work is the considered function class ($\alpha$-semi-convex), which can be seen as restrictive.
One possible extension would be weakly convex functions \citep[Sec. 1]{Davis_Drusvyatskiy2019}.
% \textbf{Broader impact.}
We deeply believe that the high quality code provided will benefit to practitioners, and ease the use of non-convex penalties for real world problems, from neuroimaging to genomics.
We proposed an optimization algorithm and do not see potential negative societal impacts.
\section*{Acknowledgements}
The experiments were ran on the CBP cluster of ENS de Lyon  \citep{quemener2013sidus}.
QB would like to thank Samsung Electronics Co., Ldt. for funding this research.
GG is supported by an IVADO grant.

\clearpage

\bibliographystyle{plainnat}
\bibliography{biblio_flashcd}

\newpage
%! TEX root=../main.tex

%%%%%%%%%%%%%%%%%%%%%%%%%%%%%%%%%%%%%%%%%%%%%%%%%%%%%%%%%%%%
\section*{Checklist}

\begin{enumerate}

\item For all authors...
\begin{enumerate}
  \item Do the main claims made in the abstract and introduction accurately reflect the paper's contributions and scope?
    \answerYes{}
  \item Did you describe the limitations of your work?
    \answerYes{See limitations paragraph}
  \item Did you discuss any potential negative societal impacts of your work?
    \answerYes{}
  \item Have you read the ethics review guidelines and ensured that your paper conforms to them?
    \answerYes{}
\end{enumerate}

\item If you are including theoretical results...
\begin{enumerate}
  \item Did you state the full set of assumptions of all theoretical results?
  \answerYes{See \Cref{prop:finite_identification,prop:dl_fixed_point,prop:acc_lin_conv}.}
        \item Did you include complete proofs of all theoretical results?
    \answerYes{See \Cref{app:missing_proofs}.}
\end{enumerate}

\item If you ran experiments...
\begin{enumerate}
  \item Did you include the code, data, and instructions needed to reproduce the main experimental results (either in the supplemental material or as a URL)?
    \answerYes{}
  \item Did you specify all the training details (e.g., data splits, hyperparameters, how they were chosen)?
    \answerYes{See \Cref{sec:experiments}.}
        \item Did you report error bars (e.g., with respect to the random seed after running experiments multiple times)?
    \answerNA{}
        \item Did you include the total amount of compute and the type of resources used (e.g., type of GPUs, internal cluster, or cloud provider)?
    \answerNA{}
\end{enumerate}

\item If you are using existing assets (e.g., code, data, models) or curating/releasing new assets...
\begin{enumerate}
  \item If your work uses existing assets, did you cite the creators?
    \answerYes{In particular we acknowledge the \text{python} ecosystem.}
  \item Did you mention the license of the assets?
    \answerNA{}
  \item Did you include any new assets either in the supplemental material or as a URL?
    \answerYes{}
  \item Did you discuss whether and how consent was obtained from people whose data you're using/curating?
   \answerNA{}
  \item Did you discuss whether the data you are using/curating contains personally identifiable information or offensive content?
   \answerNA{}
\end{enumerate}

\item If you used crowdsourcing or conducted research with human subjects...
\begin{enumerate}
  \item Did you include the full text of instructions given to participants and screenshots, if applicable?
    \answerNA{}
  \item Did you describe any potential participant risks, with links to Institutional Review Board (IRB) approvals, if applicable?
    \answerNA{}
  \item Did you include the estimated hourly wage paid to participants and the total amount spent on participant compensation?
    \answerNA{}
\end{enumerate}

\end{enumerate}

%TODO should checklist be removed?

%%%%%%%%%%%%%%%%%%%%%%%%%%%%%%%%%%%%%%%%%%%%%%%%%%%%%%%%%%%%

\clearpage
\appendix
%! TEX root=../main.tex

\onecolumn
%%%%%%%%%%%%%%%%%%%%%%%%%%%%%%%%%%%%%%%%%%%%%%%%%%%%%%%%%%%%%%%
%%%%%%%%%%%%%%%%%%%%%%%%%%%%%%%%%%%%%%%%%%%%%%%%%%%%%%%%%%%%%%%
\section{Algorithms}
\label{app:algorithms}
%%%%%%%%%%%%%%%%%%%%%%%%%%%%%%%%%%%%%%%%%%%%%%%%%%%%%%%%%%%%%%%
%%%%%%%%%%%%%%%%%%%%%%%%%%%%%%%%%%%%%%%%%%%%%%%%%%%%%%%%%%%%%%%
\hspace*{2mm}
\begin{minipage}{0.48\linewidth}
\begin{algorithm}[H]
    \SetKwInOut{Init}{init}
    \SetKwInOut{Input}{input}
    \setcounter{AlgoLine}{0}
    \Input{
        $X \in \bbR^{n \times p},
        \beta \in \bbR^p,
        X\beta \in \bbR^n,
        L \in \bbR^p,
        \mathrm{ws} \subset [p]$}
    \caption{Coordinate descent epoch}\label{alg:cd_epoch}
    \For{$j \in \mathrm{ws}$}{
        $\beta_{\mathrm{old}} \leftarrow \beta_j$
        \tcp*{$\bigo(1)$}

        $\beta_j \leftarrow
        \prox_{g_j / L_j} \left (
            \beta_j - \frac{1}{L_j} X_{:j}^\top\nabla F(X\beta)
         \right )$
         \tcp*{$\bigo(n)$}

        $X\beta \pluseq (\beta_j - \beta_{\mathrm{old}}) X_{:j}$
        \tcp*{$\bigo(n)$}
    }
  \Return{$\beta$}
\end{algorithm}
\end{minipage}
\hspace*{5mm}
\begin{minipage}{0.48\linewidth}
\begin{algorithm}[H]
    \caption{Anderson extrapolation}\label{alg:extrapolation}
    \SetKwInOut{Init}{init}
    \setcounter{AlgoLine}{0}
    \Init{
        $
        \beta^{(0)}, \dots, \beta^{(M)} \in \bbR^{|\mathrm{ws}| \times (M+1)}$}

        $U = [\beta^{(1)} - \beta^{(0)}, \ldots, \beta^{(M)} - \beta^{(M - 1)}] \in \bbR^{|\mathrm{ws}| \times M}$

        $c = (U^\top U)^{-1} \mathbf{1}_M  \in \bbR^M$
        \tcp*{$\bigo(M^2 |\mathrm{ws}| + M^3)$}

        $c \diveq \mathbf{1}_M^\top c$
        \tcp*{$\bigo(M)$}

        $\beta^{\mathrm{extr}} = \sum_{i=1}^M c_i \beta^{(i)}
        \in \bbR^{|\mathrm{ws}|}
        $
        \tcp*{$\bigo(M |\mathrm{ws}|)$}

    \Return{$\beta^{\mathrm{extr}}$}
\end{algorithm}
\end{minipage}

%%%%%%%%%%%%%%%%%%%%%%%%%%%%%%%%%%%%%%%%%%
\section{Proofs and propositions}
\label{app:missing_proofs}
%%%%%%%%%%%%%%%%%%%%%%%%%%%%%%%%%%%%%%%%%%%%%%%%%%%%%%%%%%%
\subsection{Working set convergence (\Cref{prop:conv_ws})}
\label{app:sub_conv_ws}
%%%%%%%%%%%%%%%%%%%%%%%%%%%%%%%%%%%%%%%%%%%%%%%%%%%%%%%%%%%
\begin{proof}
    Since $\mathcal{W}_t \subset \mathcal{W}_{t+1}$ after at most $p$ iterations, the working set is made of all the $p$ features.
    If \Cref{alg:skglm} stops when $|\mathcal{W}_t|< p$, then \citet[Thm. 3.1]{Bolte_Sabach_Teboulle2014} ensures that the inner solver converges towards a critical point of the restricted subproblem.
    Moreover, if the working set stops increasing, it means that for all $j\notin \mathcal{W}_t$,
    $\mathrm{score}_j^\partial$
    % computed in \Cref{eq:score_features} or in \Cref{eq:score_features_nncvx}
    is smaller than a given tolerance, hence satisfying the critical point condition of the global optimization problem.

    If $|\mathcal{W}_t|= p$, the inner solver is used on the full optimization problem and \citet[Thm. 3.1]{Bolte_Sabach_Teboulle2014} ensures convergence towards a critical point of the latter.%
\end{proof}
%%%%%%%%%%%%%%%%%%%%%%%%%%%%%%%%%%%%%%%%%%
\subsection{$\alpha$-semi-convexity of MCP (\Cref{prop:mpc_alpha_semicvx})}
\label{app:sub_alpha_semisvx_mcp}
%%%%%%%%%%%%%%%%%%%%%%%%%%%%%%%%%%%%%%%%%%
\begin{proof}
    Let $j \in [p]$, $\gamma > \min_j 1 / L_j$ and
    $
        g_j
        \eqdef
        \MCP_{\lambda, \gamma}: x \mapsto
        \begin{cases}
            \lambda \abs{x} - \frac{x^2}{2 \gamma} \enspace,
            &\text{if}\: \abs{x} \leq \gamma \lambda \\
                \frac{1}{2} \gamma \lambda^2 \enspace,
            &\text{if}\: \abs{x} > \gamma \lambda \enspace
        \end{cases}
    $, $\alpha = \frac{1}{2} \left (1 + \frac{1}{\gamma L_j}\right)$, and $h_j \eqdef \frac{g_j}{L_j} + \frac{\alpha}{2} \norm{\cdot}^2$.

    Since $\gamma > \min_j 1 / L_j$, $\alpha < 1$.
    Moreover, for all $x > 0$ such that, $\abs{x} \leq \gamma \lambda$,
    \begin{align}
        h_j(x)
        &=
        \frac{\lambda \abs{x}}{L_j}
        - \frac{x^2}{2 \gamma L_j}
        +  \frac{\alpha x^2}{2}
        , \: \text{thus}
        \\
        h_j'(x)
        &=
        \frac{\lambda}{L_j}
        +
        \frac{1}{2} \left (1 - \frac{1}{\gamma L_j}\right) x
        , \: \text{thus}
        % \\
        % h_j''(x)
        % &= \alpha - \frac{1}{\gamma L_j}
        % \\
        % h_j''(x)
        % &=  \frac{1}{2} \left (1 - \frac{1}{\gamma L_j}\right)
        % > 0
        \enspace .
    \end{align}
    In addition,  for all $x >0$ such that, $\abs{x} > \gamma \lambda$
    \begin{align}
        h_j(x)
        &=
        \frac{1}{2} \gamma \lambda^2
        +  \frac{\alpha x^2}{2},
        \: \text{thus}
        \\
        h_j'(x)
        &=
        \frac{1}{2} \left (1 + \frac{1}{\gamma L_j}\right) x
        % \\
        % h_j''(x)
        % &=
        % \frac{1}{2} \left (1 + \frac{1}{\gamma L_j}\right)
        \enspace .
    \end{align}
    Hence $h'$ is an increasing function on $]0, + \infty[$, and thus $h$ is convex on $]0, + \infty[$.
    In addition, h is increasing, symmetric, and continuous, thus $h$ is convex on $\bbR$, and MCP is $\alpha$-semi-convex.
\end{proof}

%%%%%%%%%%%%%%%%%%%%%%%%%%%%%%%%%%%%%%%%%%%
\subsection{Support identification for coordinate descent (\Cref{prop:finite_identification})}
\label{app:sub_suppid_cd}
%%%%%%%%%%%%%%%%%%%%%%%%%%%%%%%%%%%%%%%%%%
For exposition purposes, the proof is first provided for proximal gradient descent.
\begin{proof}
    \textbf{Proximal gradient descent.}
    Here we generalize the results of \citet[Sec. 6.2.2]{Nutini2018} to semi-convex $g_j$'s.
    The updates of proximal gradient descent read:
    \begin{align}\label{eq:update_pgd}
        \beta_j^{(k+1)}
        = \prox_{g_j / L} \left (
            \beta_j^{(k)} - \frac{1}{L}\nabla_j f(\beta^{(k)})
            \right)
        \enspace .
    \end{align}
    Let $\cS$ be the generalized support of $\hat \beta$.
    Using  \Cref{ass:non_degeneracy}, we have that for $ j \notin \cS$,
    \begin{align} \label{eq:interior_subdiff}
         - \nabla_j f(\hat \beta)
         &
         \in \mathrm{interior}(\partial g_j(\hat \beta_j)) \enspace .
    \end{align}
    Combining \Cref{eq:interior_subdiff} with the Lipschitz continuity of the gradient (\Cref{ass:f}) and the convergence of $ (\beta^{(k)})$ toward $\hat \beta$
    yields that there exists $k \in \bbN$ such that
    \begin{align}\label{eq:update_pgd_subdiff_opt}
        L( \beta_j^{(k)}
         - \hat \beta_j)
         - \nabla_j f(\beta^{(k)})
             \in \partial g_j(\hat \beta_j)
         \enspace .
     \end{align}
     Since $g_j / L$ is $\alpha$-semi-convex with $\alpha < 1$, \Cref{eq:update_pgd_subdiff_opt} is equivalent to
     \begin{align}\label{eq:update_pgd_opt}
        \hat \beta_j
        = \prox_{g_j / L} \left (
            \beta_j^{(k)} - \frac{1}{L}\nabla_j f(\beta^{(k)})
            \right)
        \enspace .
    \end{align}
    By uniqueness of the proximity operator (direct consequence of \Cref{ass:alpha_semi_cvx}), \Cref{eq:update_pgd,eq:update_pgd_opt} yield that there exists $K \in \bbN$ such that for all $k \geq K$, $\beta_j^{(k)} = \hat \beta_j$.

    The proof for coordinate descent is similar and can be found in \Cref{app:sub_suppid_cd}.
\end{proof}

We prove the support identification for the coordinate descent algorithm \Cref{prop:finite_identification}.

\begin{proof}
    \textbf{Proximal coordinate descent.}
    Let us denote by $\beta^{(k, j)}$ the update at the epoch $k$ and changing the coordinate $j$ with the convention that $\beta^{(k, 0)} = \beta^{(k)}$ and $\beta^{(k, p)} = \beta^{(k+1)}$.
    An update of proximal coordinate descent reads
    \begin{align}\label{eq:update_pcd}
        \beta^{(k, j)} = \prox_{g_j/L_j}\left( \beta^{(k, j-1)} - \frac{1}{L_j}\nabla_j f(\beta^{(k, j-1)})\right) \enspace .
    \end{align}
    % which implies that
    % \begin{align}\label{eq:update_pcd_subdiff}
    %     L_j(\beta^{(k, j-1)} - \beta^{(k, j)}) - \nabla_j f(\beta^{(k, j)}) \in \partial g_j(\beta^{(k, j)}) \enspace .
    % \end{align}
    %
    Let $\cS$ be the generalized support of $\hat \beta$, a critical point of \Cref{pb:generic_pb}.

    Using  \Cref{ass:non_degeneracy}, we have that for $ j \notin \cS$,
    \begin{align} \label{eq:interior_subdiff2}
        - \nabla_j f(\hat \beta)
        &
        %  \in \ri \partial g_j(\hat \beta_j)
        \in \mathrm{interior}(\partial g_j(\hat \beta_j)) \enspace .
        %  \\
        %  - \nabla_j f(\beta^{(k)})
    \end{align}
    Combining \Cref{eq:interior_subdiff2} with the Lipschitz continuity of the gradient (\Cref{ass:f}) and the convergence of $ (\beta^{(k)})$ toward $\hat \beta$
    yields that there exists $k \in \bbN$ such that
    \begin{align}\label{eq:update_pgd_subdiff_opt2}
        L_j( \beta_j^{(k, j-1)}
        - \hat \beta_j)
        - \nabla_j f(\beta^{(k, j-1)})
            \in \partial g_j(\hat \beta_j)
        \enspace .
    \end{align}
    Since $g_j / L$ is $\alpha$-semi-convex with $\alpha < 1$, \Cref{eq:update_pgd_subdiff_opt2} is equivalent to
    \begin{align}\label{eq:update_pgd_opt2}
        \hat \beta_j
        = \prox_{g_j / L_j} \left (
            \beta_j^{(k, j-1)} - \frac{1}{L_j}\nabla_j f(\beta^{(k, j-1)})
            \right)
        \enspace .
    \end{align}
    By uniqueness of the proximity operator (direct consequence of \Cref{ass:alpha_semi_cvx}), \Cref{eq:update_pcd,eq:update_pgd_opt2} yield that there exists $K \in \bbN$ such that for all $k \geq K$, $\beta_j^{(k)} = \hat \beta_j$.
\end{proof}

%%%%%%%%%%%%%%%%%%%%%%%%%%%%%%%%%%%%%%%%%%
\subsection{Local linear convergence (\Cref{prop:dl_fixed_point})}
\label{app:sub_local_lin_conv}
%%%%%%%%%%%%%%%%%%%%%%%%%%%%%%%%%%%%%%%%%
%
Here we extend in the proof of local linear convergence of coordinate descent from \citealt{Klopfenstein2020} to the $\alpha$-semi-convex case.
This property will be useful to show \Cref{prop:acc_lin_conv}.
\begin{proposition}\label{prop:dl_fixed_point}
    \sloppy
    Consider a critical point $\hat \beta$ and suppose
    \begin{enumerate}[itemsep=-2pt,topsep=-2pt]
        \item \Cref{ass:f,ass:g,ass:non_empty,ass:alpha_semi_cvx} hold.
        \item The sequence $(\beta^{(k)})_{k\geq 0}$ generated by coordinate descent (\Cref{alg:subproblem} without extrapolation) converges to a critical point $\hat \beta$.
        \item Suppose \Cref{ass:non_degeneracy,ass:locally_c2,ass:local_str_cvx} hold for $\hat \beta$.
    \end{enumerate}
    Then there exists
    $K \in \mathbb{N}$, and a $\mathcal{C}^1$ function
    $\psi: \mathbb{R}^{|\cS|} \to \mathbb{R}^{|\cS|}$
    such that, for all $k \in \mathbb{N}, k \geq K$:
    \begin{align}
        \beta_j^{(k)}
        =
        \hat \beta_j
        \text{ , for all $j \in \cS^c$, }
        \text{ and }
        \,
        % \\
        \beta_{\cS}^{(k+1)}
        -
        \hat \beta_{\cS}
        =
        \cJ \psi (\hat \beta_{\cS})
        (\beta_{\cS}^{(k)} -
        \hat \beta_{\cS})
        + \bigo(\normin{\beta_{\cS}^{(k)} - \hat \beta_{\cS}}^2 )
        \enspace ,
        \nonumber
    \end{align}
    and
    % \begin{align}
        $\rho \left (
            \cJ \psi (\hat \beta_{\cS})
        \right )
        <
        1
        \enspace ,$
    % \end{align}
    where $\cJ \psi$ is the Jacobian of $\psi$,
    and $\rho$ its spectral radius.
\end{proposition}
%
% Proof of \Cref{prop:dl_fixed_point} is in \Cref{app:sub_local_lin_conv}.
% Note that \Cref{ass:local_str_cvx} implies an $1/$ L
Note that under the hypothesis that $\Phi$ is $1/2$-{\L}ojasiewicz, local linear convergence can be provided by \citet[Remark 3.4]{Bolte_Sabach_Teboulle2014}.

\begin{proof}

Let $\gamma_j = 1 / L_j$.
Let $\cS$ be the generalized support.
Its elements are numbered as follows:
    \begin{math}
        \cS = \{j_1, \hdots, j_{|\cS|}\}
    \end{math}.
    We also define $\pi : \bbR^{|\cS|} \to \bbR^p$ for all $\beta_{\cS} \in \bbR^{|\cS|}$ and all $j \in \cS$ by
    \begin{align}
        \pi(\beta_{\cS})_j =
        \begin{cases}
            \beta_j,
            & \text{if } j \in \cS \\
            \hat \beta_j,
            & \text{if } j \in \cS^c
            \enspace,
        \end{cases}
        \nonumber
    \end{align}
    and for all
    $s \in [|\cS|]$, $\cP^{(s)} : \bbR^{|\cS|} \to \bbR^{|\cS|}$
    is defined for all
    $u \in \bbR^{|\cS|}$
    and all $s' \in [|S|]$ by
    \begin{align}
        \left(
            \cP^{(s)} (u)
        \right)_{s'}
        =
        \begin{cases}
            u_{s'} &
            \text{if } s \neq s' \\
            \prox_{\gamma_{j_s} g_{j_s}}
            \left( u_{s} - \gamma_{j_s}\nabla_{j_s} f(\pi(u)) \right) &
            \text{if } s = s' \enspace.
        \end{cases}
        \nonumber
    \end{align}
    Once the model is identified (\Cref{prop:finite_identification}), we have that there exists $K\geq 0$ such that for all $k\geq K$
    \begin{align}
        \beta^{(k)}_{\cS^c}
        &=
        \hat \beta_{\cS^{c}}
        \quad \text{and} \quad
        \nonumber
        \\
        \beta^{(k+1)}_\cS
        ={\psi} ( \beta^{(k)}_\cS)
        & \eqdef
        \cP^{(|\cS|)} \circ \hdots \circ \cP^{(1)}(\beta^{(k)}_\cS)
        \enspace .
        \label{eq:def_fixed_point_suppid}
    \end{align}
The proof of \Cref{prop:dl_fixed_point} follows three steps:
\begin{itemize}
    \item First we show that the fixed-point operator ${\psi}$ is differentiable at $\hat \beta_{\cS}$ (\Cref{lemma:diff_prox}).
    \item Then we show that the Jacobian spectral radius of ${\psi}$ is strictly bounded by one (\Cref{lemma:bound_jacobian_suppid}).
    Proof of \Cref{lemma:bound_jacobian_suppid} relies on \Cref{lemma:M_sym_def_pos,lemma:norm_Bj,lemma:proj_Bj,lemma:ortho_xj}.
    \item  Finally we conclude by a seconder order Taylor expansion of the fixed-point operator $\psi$.
    % to local linear convergence by applying \citealt[Theorem 1, Section 2.1.2]{Polyak1987}.
\end{itemize}

\begin{lemma}[Differentiability of the fixed-point operator]\label{lemma:diff_prox}
    The fixed-point operator ${\psi}$ is
    twice
    differentiable at $\hat \beta_{\cS}$ with Jacobian:
    %%%%%%%%%%%%%%%%%%%%%%%%%%%%%%%%%%%%%%%%%%
    \begin{align}
        \cJ{\psi}(\hat \beta_{\cS})
        & =
        M^{-1/2}
        \underbrace{
            (\Id - B^{(|\cS|)}) \hdots (\Id - B^{(1)})
        }_{\eqdef A} M^{1/2}
        \enspace ,
        \label{eq:formula_fixed_point_jac}
    \end{align}
    %%%%%%%%%%%%%%%%%%%%%%%%%%%%%%%%%%%%%%%%%%
    with
    %%%%%%%%%%%%%%%%%%%%%%%%%%%%%%%%%%%%%%%%%%
    \begin{align}
        M & \eqdef
        \nabla_{\cS, \cS}^{2} f(\hat \beta)
        +
        \nabla_{\cS, \cS}^{2} g(\hat \beta)
        \in \bbR^{|\cS| \times |\cS|} \enspace,
        \text{ and }
        \label{eq:def-mat-M}
        \\
        B^{(s)}
        & \eqdef
        M_{: s}^{1/2}
        \frac{\gamma_{j_s}}{1 + \gamma_{j_s} \nabla^2 g_{j_s}(\hat \beta_{j_s})}
        M_{: s}^{1/2\top}
        \in \bbR^{|\cS| \times |\cS|} \enspace .
        \label{eq:def-mat-Bj}
    \end{align}
    %%%%%%%%%%%%%%%%%%%%%%%%%%%%%%%%%%%%%%%%%%
    % and $\hat z \eqdef \hat \beta - \gamma \odot \nabla f(\hat \beta) \in \bbR^p$.
\end{lemma}

%%%%%%%%%%%%%%%%%%%%%%%%%%%%%%%%%%%%%%%%%%%%%%%%%%%%%%%%%%%%%%%%%%%%%%%%%%%%%
\begin{proof}[{\Cref{lemma:diff_prox}}]
    Let $j \in \cS$, $\hat \beta_j = \prox_{\gamma_j g_j}(\hat z_j)$,
    since $g_j$ is $\cC^1$ at $\hat \beta_j$ (\Cref{ass:locally_c2}), we have
    %
    % we have $\frac{1}{\gamma_j}(v-u)\in \partial g_j(u)$ becomes
    %
    % \begin{math}
        $z_j  = (\Id + \gamma_j g_j')(\hat \beta_j) \eqdef \phi(\hat \beta_j)$, and $\hat \beta_j = \phi^{(-1)}(\hat z_j)$
        \nonumber
    % \end{math}
    Since $j \in \cS$, $g_j$ is of class $\cC^3$ at
    $\hat \beta$ (\Cref{ass:locally_c2}),
    $\phi$ is $\cC^2$ at $\hat \beta_j$.
    Moreover $\phi'(\hat \beta_j) = 1 + \gamma_j g_j(\hat \beta_j) > 0$ (using \Cref{ass:alpha_semi_cvx}).
    Hence the inverse function theorem yields
    $\phi^{(-1)}$ is $\cC^1$ at $ \hat z_j \eqdef \phi(\hat \beta_j)$, and
    \begin{align}
        \prox'(\hat z_j)
        &=
        \phi^{(-1)'}(\hat z_j)
        \\
        &=
        \frac{1}{\phi'(\phi^{(-1)}(\hat z_j))}
        \\
        &=
        \frac{1}{1 + \gamma_j g_j''(\prox(\hat z_j) ) }
        \\
        &=
        \frac{1}{1 + \gamma_j g_j''(\hat \beta_j)}
        \enspace .
        \label{eq:value_diff_prox}
    \end{align}
    It follows that $\cP^{(s)}$ is $\cC^2$ at $\hat \beta_{\cS}$.
    For all $s \in [|\cS|]$, $\cP^{(s)}$ is $\cC^2$ at $\hat \beta_\cS$.
    In addition, $\cP^{(s)} (\hat \beta_\cS) = \hat \beta_\cS$, thus
    $\psi\eqdef\cP^{(|\cS|)} \circ \hdots \circ\cP^{(1)}$
     is also $\cC^1$ at $\hat \beta_\cS$.
%%%%%%%%%%%%%%%%%%%%%%%%%%%%%%%%%%%%%%%%%%%%%%%%%%%%%%%%%%%
    To compute, the Jacobian of $\cP^{(s)}$ at $\hat \beta_\cS$, let us first notice that
    \begin{align}
        \cJ\cP^{(s)}(\hat \beta_\cS)^\top =
        \left (
            \begin{array}{c|c|c|c|c|c|c}
                e_1
                & \hdots
                & e_{s-1}
                & v_s
                & e_{s+1}
                & \hdots
                & e_{|\cS|}
            \end{array}
        \right) \enspace ,
    \end{align}
    where $v_s = \prox'_{\gamma_{j_s} g_{j_s}}
        \left(
            \hat z_{j_s}
        \right)
        \left(
            e_{j_s} - \gamma_{j_s} \nabla^{2}_{j_s,:} f(\hat \beta)
        \right)$
        and
    $\hat z_j = \hat \beta_j - \gamma_j \nabla_j f(\hat \beta)$.
    This matrix can be rewritten
    \begin{align}
        \cJ{\cP^{(s)}}(\hat \beta_\cS)
        &= \Id_{|S|} - e_s e_s^{\top}
        +  \prox'_{\gamma_{j_s} g_{j_s}}
            \left(\hat z_{j_s}\right)
            \left (
                e_s e_s^\top - \gamma_{j_s} e_s e_s^\top \nabla^{2}f(\hat \beta)
            \right )
            \nonumber \\
        &= \Id_{|S|} - e_se_s^{\top} \frac{\gamma_{j_s} }{1 + \gamma_{j_s}g_j''(\hat \beta_{j_s})}
            \left (
                \nabla_{\cS, \cS} g(\hat \beta)
                + \nabla_{\cS, \cS}^{2}f(\hat \beta)
            \right)
            \nonumber \\
        &= \Id_{|S|} - e_se_s^{\top} \frac{\gamma_{j_s} }{1 + \gamma_{j_s}g_j''(\hat \beta_{j_s})}
             M
            \nonumber \\
        &= M^{-1/2}
        \left(\Id_{|S|} - M^{1/2} e_s e_s^{\top}
        \frac{\gamma_{j_s} }{1 + \gamma_{j_s}g_j''(\hat \beta_{j_s})}
         M^{1/2}\right)
        M^{1/2}
        \nonumber \\
        & = M^{-1/2}
            \left(
                \Id_{|S|} - B^{(s)}
            \right)M^{1/2}
            \enspace ,
        \nonumber
    \end{align}
    where
    \begin{align}
        M \eqdef
        \nabla_{\cS, \cS}^{2} f(\hat \beta)
        +  \nabla_{\cS, \cS}^{2} g(\hat \beta)
        \in \bbR^{|\cS| \times |\cS|}
        \enspace ,
    \end{align}
    and
    \begin{align}
        B^{(s)}
        \eqdef
        M^{1/2}_{: s}
        \frac{\gamma_{j_s} }{1 + \gamma_{j_s}g_j''(\hat \beta_{j_s})}
        M_{: s}^{1/2\top}
        \in \bbR^{|\cS| \times |\cS|}
        .
        \nonumber
    \end{align}
    The chain rule yields
    \begin{align}
        \cJ{\psi}(\hat \beta_{\cS})
        & =
        \cJ\mathcal{P}^{(|\cS|)}(\hat \beta_{\cS})
        \cJ\mathcal{P}^{(|\cS|-1)}(\hat \beta_{\cS})\hdots
        \cJ\mathcal{P}^{(1)}(\hat \beta_{\cS})
        \nonumber
        \\
        & =
        M^{-1/2}
            \underbrace{
                (\Id - B^{(|\cS|) })
                (\Id - B^{(|\cS| - 1) })
                \hdots
                (\Id - B^{(1)})
            }_{\eqdef A}
        M^{1/2}
        \nonumber
        \enspace .
    \end{align}

\end{proof}
%%%%%%%%%%%%%%%%%%%%%%%%%%%%%%%%%%%%%%%%%%%%%%%%%%%%%%%%%%%%%%%%%%%%%%%%%%%%%
The following lemmas (\Cref{lemma:M_sym_def_pos,lemma:norm_Bj,lemma:proj_Bj,lemma:ortho_xj}) aim at showing that the spectral radius of the Jacobian of the fixed-point operator $\psi$ is strictly bounded by one (\Cref{lemma:bound_jacobian_suppid}).
\begin{lemma}
    \begin{lemmaenum}[topsep=4pt,itemsep=4pt,partopsep=4pt,parsep=4pt]
    %   \item For all $j \in S$,
    %   $\prox'_{\gamma_j g_j }(\hat z_j) = \frac{1}{1 + \gamma_j \nabla^2 g_j(\hat \beta_j)} > 0$,
    %   $\hat z \eqdef \hat \beta - \gamma \odot \nabla f(\hat \beta) \in \bbR^p$.
    %   \label{lemma:diff_prox_z}
      %%%%%%%%%%%%%%%%%%%%%%%%%%%%%%%%%%%%%%%%%%%%%%%%%%%%%%
      \item The matrix $M$ defined in \Cref{eq:def-mat-M} is symmetric definite positive.
      \label{lemma:M_sym_def_pos}
      %%%%%%%%%%%%%%%%%%%%%%%%%%%%%%%%%%%%%%%%%%%%%%%%%%%%%%
      \item For all $s\in [|\cS|]$, the spectral radius of the matrix $B^{(s)}$ defined in \Cref{eq:def-mat-Bj} is bounded by 1, \ie $\normin{B^{(s)}}_2 \leq 1$.
      \label{lemma:norm_Bj}
      %%%%%%%%%%%%%%%%%%%%%%%%%%%%%%%%%%%%%%%%%%%%%%%%%%%%%%
      \item For all $s\in [|\cS|]$, $B^{(s)} / \normin{B^{(s)}}$ is an orthogonal projector onto $\Span(M_{: s}^{1/2})$.
          \label{lemma:proj_Bj}
      %%%%%%%%%%%%%%%%%%%%%%%%%%%%%%%%%%%%%%%%%%%%%%%%%%%%%%
      \item For all $s\in [|\cS|]$ and for all $u \in \bbR^S$,
      if $\normin{(\Id - B^{(s)})u} = \norm{u}$ then
      $u \in \Span(M_{: s}^{1/2})^\perp$
      and
      $(\Id - B^{(s)})u = u$.
      \label{lemma:ortho_xj}
      %%%%%%%%%%%%%%%%%%%%%%%%%%%%%%%%%%%%%%%%%%%%%%%%%%%%%%
      \item The spectral radius of the Jacobian $\cJ{\psi}(\hat \beta_{\cS})$ of the fixed-point operator $\psi$ is bounded by 1
      \begin{align}
          \rho(\cJ{\psi}(\hat \beta_{\cS}))
          <
          1
          \enspace .
      \end{align}
      \label{lemma:bound_jacobian_suppid}
      %%%%%%%%%%%%%%%%%%%%%%%%%%%%%%%%%%%%%%%%%%%%%%%%%%%%%%
    \end{lemmaenum}
  \end{lemma}
%%%%%%%%%%%%%%%%%%%%%%%%%%%%%%%%%%%%%%%%%%%%%%%%%%%%%%%%%%%%%%%%%%%%%%%%%%%%%
\begin{proof}[{\Cref{lemma:M_sym_def_pos}}]
    Using \Cref{eq:def-mat-M} yields
    \begin{align}
        M &=
        % \nabla_{\cS, \cS}^{2} f(\hat \beta)
        % + \diag(\nabla^2 g_j(\hat \beta_j))_{j \in \cS}
        % \\
        % &=
        \nabla_{\cS, \cS}^{2} f(\hat \beta)
        + \nabla_{\cS, \cS}^2 g (\hat \beta) \succ 0 \quad \text{(using \Cref{ass:local_str_cvx})} \enspace .
    \end{align}
\end{proof}
%%%%%%%%%%%%%%%%%%%%%%%%%%%%%%%%%%%%%%%%%%%%%%%%%%%%%%%%%%%%%%%%%%%%%%%%%%%%%
%%%%%%%%%%%%%%%%%%%%%%%%%%%%%%%%%%%%%%%%%%%%%%%%%%%%%%%%%%%%%%%%%%%%%%%%%%%%%
\begin{proof}[\Cref{lemma:norm_Bj}]
    $B^{(s)}$ is a rank one matrix which is the product of
    $ \dfrac{\gamma_{j_s}}{1
    +
    \gamma_{j_s}\nabla^2 {g_{j_s}} (\beta_{j_s})}
    % \nabla\prox_{\gamma_{j_s}g_j}(\hat z_{j})
    M_{:s}^{1/2}$
    and
    $M_{: s}^{1/2\top}$,
    its non-zeros eigenvalue is thus given by
    \begin{align}
        \normin{B^{(s)}}_2
        &=
        \left |
        M^{1/2\top}_{: s}
        % \gamma_{j_s}\nabla\prox_{\gamma_{j_s} g_{j_s}}(\hat z_j)_s
            \dfrac{\gamma_{j_s}}{1
        +
        \gamma_{j_s}\nabla^2 {g_{j_s}} (\beta_{j_s})}
         M_{:s}^{1/2}
         \right |
        \nonumber \\
        &=
        \left |
            \dfrac{\gamma_{j_s}}{1
        +
        \gamma_{j_s}\nabla^2 {g_{j_s}} (\beta_{j_s})} M_{s, s}
        \right |
        \nonumber \\
        & =
        \left |
        \dfrac{\gamma_{j_s}}{1
        +
        \gamma_{j_s}\nabla^2 {g_{j_s}} (\beta_{j_s})}
        \left ( \nabla_{j_s, j_s}^2 f(\beta) + \nabla^2 {g_{j_s}} (\beta_j)  \right)
        \right | \enspace.
        \nonumber
        \\
        & =
        \dfrac{
           \overbrace{\gamma_{j_s} \nabla_{j_s, j_s}^2 f(\beta)}^{\leq 1}
            +
            \gamma_{j_s}\nabla^2 {g_{j_s}} (\beta_{j_s})
        }{
            1
            +
            \gamma_{j_s}\nabla^2 {g_{j_s}} (\beta_{j_s})
        }
        \leq 1 \enspace.
        \nonumber
    \end{align}
\end{proof}
%%%%%%%%%%%%%%%%%%%%%%%%%%%%%%%%%%%%%%%%%%%%%%%%%%%%%%%%%%%%%%%%%%%%%%%%ùùù
\begin{proof}[\Cref{lemma:ortho_xj}]
    Let $s \in \cS$,
    \begin{align}
        \Id - B^{(s)}
        &=
        \Id - \normin{B^{(s)}} \frac{B^{(s)}}{\normin{B^{(s)}}} \nonumber
        =
        (1 - \normin{B^{(s)}}) \Id
        + \normin{B^{(s)}}_2 \Id
        - \normin{B^{(s)}}_2 \frac{B^{(s)}}{\normin{B^{(s)}}_2} \nonumber \\
        &=
        (1 - \normin{B^{(s)}}) \Id
        + \normin{B^{(s)}}
        \underbrace{\left ( \Id - \frac{B^{(s)}}{\normin{B^{(s)}}_2} \right )}_{\text{projection onto } M_{: s}^{1/2\perp}} \enspace . \label{eq:almost_proj}
    \end{align}
    We will prove \Cref{lemma:ortho_xj} with a proof by absurd.
    Suppose that there exists $u \notin \Span( M_{: s}^{1/2})^\perp$ such that $\normin{(\Id - B^{(s)})u} = \norm{u}$.

    There exists
    $\alpha \neq 0$,
    $u_{ M_{: s}^{1/2\perp}} \in \Span( M_{: s}^{1/2})^\perp$  such that
    \begin{align} \label{eq:decomposition_x}
        u = \alpha M_{: s}^{1/2} + u_{ M_{: s}^{1/2\perp}} \enspace .
    \end{align}
    Combining \Cref{eq:almost_proj,eq:decomposition_x} yields:
    \begin{align}
        (\Id - B^{(s)}) u
        &=
        (1 - \normin{B^{(s)}}_2) u
        + \normin{B^{(s)}}_2 u_{ M_{: s}^{1/2\perp}}
        \nonumber\\
        \normin{(\Id - B^{(s)}) u}
        &\leq
        \underbrace{
            |1 - \normin{B^{(s)}}_2|
            }_{= (1 - \normin{B^{(s)}}_2)}
        \normin{u}
        + \normin{B^{(s)}}_2
        \underbrace{
            \normin{u_{ M_{: s}^{1/2\perp}}}
        }_{< \normin{u}}
        \nonumber
         < \normin{u}
        \nonumber
        \enspace ,
    \end{align}
    which contradicts the supposition $\normin{(\Id - B^{(s)})u} = \norm{u}$.
    Thus $u \in \Span( M_{: s}^{1/2})^\perp$
    and $(\Id - B^{(s)}) u = u$.
\end{proof}
%%%%%%%%%%%%%%%%%%%%%%%%ùù
\begin{proof}
    [\Cref{lemma:bound_jacobian_suppid}].
        Let $u \in \bbR^s$ such that
            $\normin{(\Id_{|\cS|} - B^{(|\cS|)}) \dots (\Id_{|\cS|} - B^{(1)})u}
            =
            \norm{u}
            \enspace .$
        Since
        \begin{align}
            \normin{(\Id_{|\cS|} - B^{(|\cS|)}
            \dots
            (\Id_{|\cS|} - B^{(1)})}_2
            \leq
            \underbrace{
                \normin{(\Id_{|\cS|} - B^{(|\cS|)})}_2
                }_{\leq 1}
            \times \dots \times
            \underbrace{\normin{(\Id_{|\cS|} - B^{(1)})}_2}_{\leq 1}
            \nonumber \enspace ,
        \end{align}
        we have for all $j \in \cS$,
        \begin{math}
            \normin{(\Id_{|\cS|} - B^{(s)})u} = \norm{u}
        \end{math}.
        One can thus successively apply \Cref{lemma:ortho_xj} which yields
        \begin{math}
        u
        \in
        \bigcap_{s \in [|\cS|]}
        \Span{ M_{: s}^{1/2}}^\perp \Leftrightarrow
            u \in \Span ( M_{: 1}^{1/2}, \dots, M_{: |\cS|}^{1/2} )^\perp \nonumber \enspace .
        \end{math}
        Moreover $M^{1/2}$ has full rank, thus $u=0$ and
            $\normin{(\Id_{|\cS|} - B^{(|\cS|)}) \dots (\Id_{|\cS|} - B^{(1)})}_2 < 1 \nonumber \enspace .$
    From \Cref{lemma:bound_jacobian_suppid}, $\normin{A}_2<1$. Moreover $A$ and
    $\cJ{\psi}(\hat \beta_{\cS})$ are similar matrices (\Cref{eq:formula_fixed_point_jac}), then
    $\rho(\cJ{\psi}(\hat \beta_{\cS})) = \rho(A) < 1$.
    \end{proof}
    % %
    % Then all conditions are met for a second order Taylor
    % Then all conditions (\Cref{lemma:diff_prox,lemma:bound_jacobian_suppid}) are met to apply
    % \cite[Theorem 1, Section 2.1.2]{Polyak1987} which proves the local linear convergence.
\end{proof}

%%%%%%%%%%%%%%%%%%%%%%%%%%%%%%%%%%%%%%%%%%%%%%%%%%%
\subsection{Local acceleration (\Cref{prop:acc_lin_conv})}\label{app:local_acc}
%%%%%%%%%%%%%%%%%%%%%%%%%%%%%%%%%%%%%%%%%%%%%%%%%%%
\begin{proof}
    Since \Cref{ass:f,ass:g,ass:non_empty,ass:non_degeneracy} hold, coordinate descent achieves finite time support identification (\Cref{prop:finite_identification}): there exists $l_0$ such that for all $l \geq l_0$,  $\beta^{(l_0 M + 1)}, \dots, \beta^{(l_0 M + M)}$ shares the support of $\hat \beta$.
    Since Anderson acceleration  linearly combines iterates from $\beta^{(l_0 M + 1)}$ to $\beta^{(l_0 M + M)}$, it preserves the finite time identification property.

In addition, since \Cref{ass:non_degeneracy,ass:locally_c2,ass:local_str_cvx} hold, and the functions $f$ and $g_j$, $j \in [p]$ are piecewise quadratic (by hypothesis),
then \Cref{prop:dl_fixed_point} yields that there exists $K \in \bbN$ such that, for all $k \geq K$:
\begin{align}
    \beta_{\cS^c}^{(k)}
    &= \hat \beta_{\cS^c}
    \\
    \beta_{\cS}^{(k+1)}
    -
    \hat \beta_{\cS}
    &=
    \cJ \psi (\hat \beta_{\cS})
    (\beta_{\cS}^{(k)} -
    \hat \beta_{\cS})
    \enspace .
\end{align}
If the coordinate descent indices are picked from $1$ to $p$ and then form $p$ to $1$, then
\begin{align}
    T  \eqdef \cJ{\psi}(\hat \beta_{\cS})
    & =
    M^{-1/2}
        (\Id - B^{(1)}) \hdots (\Id - B^{(|\cS|)})
        (\Id - B^{(|\cS|)}) \hdots (\Id - B^{(1)})
    M^{1/2}
    \enspace.
    \label{eq:formula_fixed_point_jac_sym}
\end{align}
Based on \Cref{eq:formula_fixed_point_jac_sym} one can apply \citet[Prop. 4]{Bertrand_Massias2020}, which yields
$\rho(T) < 1$ and the iterates of Anderson extrapolation with parameter $M$ enjoy local accelerated convergence rate:
        \begin{align}
            \normin{\beta_{\cS}^{(k-K)} - \hat \beta_{\cS}}_B
            \leq
            \Big(
                \sqrt{\kappa(H)}
                \tfrac{2\zeta^{M-1}}{1 + \zeta^{2(M-1)}} \Big)^{(k-K) / M}
            \normin{\beta_{\cS}^{(K)} - \hat \beta_{\cS}}_B
            \enspace,
        \end{align}
with
$H \eqdef \nabla_{\cS, \cS}^{2} f(\hat \beta) + \nabla_{\cS, \cS}^{2} g(\hat \beta)$,
$\zeta \eqdef (1 - \sqrt{1 - \rho(T)}) / (1 + \sqrt{1- \rho(T)})$,
$B \eqdef ( T - \Id)^\top ( T - \Id)$.
\end{proof}

%%%%%%%%%%%%%%%%%%%%%%%%%%%%%%%%%%%%%%%%%%%%%%%%%%%
\section{Beyond $\alpha$-semi-convex penalties}\label{app:lq}
%%%%%%%%%%%%%%%%%%%%%%%%%%%%%%%%%%%%%%%%%%%%%%%%%%%
\subsection{Proposed score}
When the $g_j$'s are not convex, the distance to the subdifferential can yield uninformative priority scores. This is in particular the case for $\ell_q$-penalties, with $0<q<1$.
\begin{example}\label{ex-05}
    The subdifferential of the  $\ell_{0.5}$-norm at $0$ is:
    $\partial g (0) = \bbR$.
    Hence $0_p$ is a critical point for any $f$.
    For any $\beta$,
    \begin{align}
        \dist(-\nabla_j f (\beta), \partial g_j (0)) = 0
        \enspace ,
    \end{align}
    Thus if $\beta_j = 0$, no matter the value of $\nabla_j f(\beta)$, feature $j$ is always assigned a score of 0,
    which is not relevant to discriminate important features.
\end{example}
% Non-convex $\ell_p$-penalties are
% % of outmost importance
% important
% in sparse machine learning and signal processing \cite{chartrand2007exact,Strohmeier_Bekhti_Haueisen_Gramfort2016,wen2018survey}, it is critical to be able to handle them.
A key observation to improve this rule is that, although $0_p$ is a critical point for any $f$, coordinate descent is able to escape it (\Cref{app:cd_escapes}).
Instead of considering critical point, we consider the more restrictive condition of being a fixed point of proximal coordinate descent:
\begin{equation}
    \hat \beta_j = \prox_{g_j / L_j}
        \left (\hat \beta_j - \frac{1}{L_j} \nabla_j f(\hat \beta)
        \right ) \enspace.
\end{equation}
We propose to rely on the violation of the fixed-point equation:
\begin{equation}\label{eq:score_features_nncvx}
    \mathrm{score}_j^{\mathrm{cd}}
    =
    | \beta_j - \prox_{g_j / L_j}
        (\hat \beta_j - \nabla_j f(\hat \beta) / L_j ) | \enspace.
\end{equation}
This is in a sense a restriction of the optimality conditions, since a fixed point is a critical point (while the converse may not be true).

Because this score only relies on $\nabla f$ and $\prox_{g_j}$, which are known for the overwhelming majority of instances of \Cref{pb:generic_pb}, our working set algorithm can address all of these, while being very simple to implement.
This is in contrast with algorithm relying on duality, or on geometrical interpretations.

\begin{remark}
    Feature importance measures such as $\mathrm{score}_j^\partial$ and $\mathrm{score}_j^{\mathrm{cd}}$ have been considered in the convex case by \citet[Sec. 8]{Nutini_Schmidt_Laradji_Friedlander_Koepke15}, while studying the Gauss-Southwell greedy coordinate descent selection rule.
    However, their approach is to compute the whole score vector \eqref{eq:score_features_nncvx}, which requires a full gradient computation, in order to update a single coordinate: it is not a practical algorithm.
\end{remark}
%%%%%%%%%%%%%%%%%%%%%%%%%%%%%%%%%%%%%%%%%%%%%%%
\subsection{Coordinate descent escapes $0_p$}
\label{app:cd_escapes}
%%%%%%%%%%%%%%%%%%%%%%%%%%%%%%%%%%%%%%%%%%%%%%%%%%%
Let $f$ be a generic function satisfying \Cref{ass:f}.
Suppose that coordinate descent is run on
\begin{equation}
    \min f(\beta) + \lambda \sum_1^p \sqrt{\abs{\beta_j}} \enspace,
\end{equation}
initialized at $0_p$ (a critical point, as seen in \Cref{ex-05}).
We show that if $\lambda$ is low enough, coordinate descent escapes this point.
As coordinate descent is a descent method ($f$ is convex: the objective decreases strictly every time a coordinate's value changes), it is sufficient to show that at least one coordinate is updated.

Let $j = \argmax \abs{\nabla_j f(0_p)}$.
When comes the time for coordinate $j$ to be updated, if some coordinate's value has already changed, coordinate descent has escaped the origin.
Otherwise, since the proximal operator of $x \mapsto  \frac{\lambda}{L_j} \sqrt{x}$ is 0-valued exactly on $[-\frac{3}{2} \left (\frac{\lambda}{L_j} \right)^{2/3}, \frac{3}{2} \left( \frac{\lambda}{L_j} \right)^{2/3}]$ \citep[Table 1]{wen2018survey}, if
\begin{equation}
    \frac{1}{L_j} \abs{\nabla_j f(0_p)} > \frac{3}{2} \left (\frac{\lambda}{L_j} \right)^{2/3}, \enspace
\end{equation}
then the value of $\beta_j$ changes.
Thus, for $\lambda < \left( \frac{2}{3} \frac{\abs{\nabla_j f(0_p)}}{L_j^{1/3}} \right)^{3/2}$, coordinate descent escapes the origin.

%%%%%%%%%%%%%%%%%%%%%%%%%%%%%%%%%%%%%%%%%%%%%%%%%%%%%%%%%%%%%%%%%%%%%%ù
\section{Proximal operator of penalties in the multitask setting}\label{app:prox_multitask}
%%%%%%%%%%%%%%%%%%%%%%%%%%%%%%%%%%%%%%%%%%%%%%%%%%%%%%%%%%%%%%%%%%%%%%ù
In this section we consider a penalty on rows of matrices, i.e. letting $\phi$ be a 1 dimensional penalty, which is even, the whole penalty on $W \in \bbR^{p \times d}$ is
\begin{equation}
    g(W) = \sum_{j=1}^p \phi(\normin{W_{j:}})
\end{equation}
Since this penalty is separable, this brings us to solving:
\begin{equation}\label{eq:prox_radial}
    \argmin_{y \in \bbR^d} \frac{1}{2} \normin{y - x}^2 + \phi(\normin{y}) \enspace.
\end{equation}
\begin{proposition}
    The proximal operator of $y \mapsto \phi(\norm{y})$ is given by:
    \begin{equation}
        \prox_{\phi(\norm{\cdot})} (x) = \prox_{\phi} (\norm{x}) \frac{x}{\norm{x}}
    \end{equation}
\end{proposition}
\begin{proof}
    Notice that the minimum is necessarily attained at a point equal to $tx$, with $t \geq 0$: indeed, for any $y$, $\frac{\norm{y}}{\norm{x}} x$ yields $\phi(\norm{\frac{\norm{y}}{\norm{x}} x}) = \phi(\norm{y})$, and since:
    \begin{align}
        \norm{\frac{\norm{y}}{\norm{x}} x - x}^2 = \norm{y}^2 - 2 \norm{y} \norm{x} + \norm{x}^2 \leq \norm{y}^2 - 2 \langle y, x \rangle + \norm{x}^2 = \norm{y - x}^2 \enspace,
    \end{align}
    it achieves lower objective value in \eqref{eq:prox_radial} than $y$.
    Hence the problem transforms into a 1 dimensional one:
    \begin{align}
        \argmin_{y \in \bbR^d} \frac{1}{2}\norm{y - x}^2 + \phi(\norm{y})
        &= \left(\argmin_{t \geq 0} \frac{1}{2} (t - 1)^2 \norm{x}^2 + \phi(\abs{t} \norm{x})) \right) x \nonumber \\
        &= \left(\argmin_{t \in \bbR} \frac{1}{2} (t - 1)^2 \norm{x}^2 + \phi(t \norm{x})) \right) x  \quad \quad \quad \text{($\phi$ is even)}\nonumber \\
        &= \left(\argmin_{t \in \bbR} \frac{1}{2} (t - \norm{x} )^2 + \phi(t)) \right) \frac{x}{\norm{x}} \nonumber \\
        &= \prox_{\phi} (\norm{x}) \frac{x}{\norm{x}} \enspace.
    \end{align}
\end{proof}

%%%%%%%%%%%%%%%%%%%%%%%%%%%%%%%%%%%%%%%%%%%%%%%%%%%%%%%%%%%%%%%
\section{Supplementary experiments}
\label{app:expe_suppl}
%%%%%%%%%%%%%%%%%%%%%%%%%%%%%%%%%%%%%%%%%%%%%%%%%%%%%%%%%%%%%%%
\subsection{Datasets characteristics}
%%%%%%%%%%%%%%%%%%%%%%%%%%%%%%%%%%%%%%%%%%%%%%%%%%%%%%%%%%%%%%%
%
\begin{table}[H]
    \center
    \caption{Datasets characteristics.}
    \begin{tabular}{ccccc}
        \hline
        Datasets & \#samples $n$ & \#features $p$ & density \\
        \hline
        \emph{rcv1} & \num{20242} & \num{19959} & \num{3.6e-3} \\
        \emph{news20} & \num{19996} & \num{1355191} & \num{3.4e-4} \\
        \emph{finance} & \num{16087} & \num{4272227} & \num{1.4e-3} \\
        \emph{kdda} & \num{8407752} & \num{20216830} & \num{1.8e-6} \\
        \emph{url} & \num{2396130} & \num{3231961} & \num{3.6e-5} \\
        \hline
    \end{tabular}
    \label{table:summary_data}
\end{table}
%%%%%%%%%%%%%%%%%%%%%%%%%%%%%%%%%%%%%%%%%%%%%%%%%%%%%%%%%%%%%%%
%%%%%%%%%%%%%%%%%%%%%%%%%%%%%%%%%%%%%%%%%%%%%%%%%%%%%%%%%%%%%%%
\subsection{ADMM comparison}
%%%%%%%%%%%%%%%%%%%%%%%%%%%%%%%%%%%%%%%%%%%%%%%%%%%%%%%%%%%%%%%
    ADMM can solve a larger range of optimization problems
    than CD \citep[Eq. 3.1]{Boyd_Parikh_Chu2011}.
    Yet, for the Lasso, ADMM requires solving a $p \times p$ linear system at each primal iteration. This is too costly: ADMM is usually not included in Lasso benchmarks (e.g. \citealt{Johnson_Guestrin15}).
    Our algorithm outperforms the implementation of \citet{Poon_Liang2019} as visible on \Cref{fig:duality_gap_admm}.
    \begin{figure}[H]
        \centering
            \centering
            \includegraphics[width=0.4\linewidth]{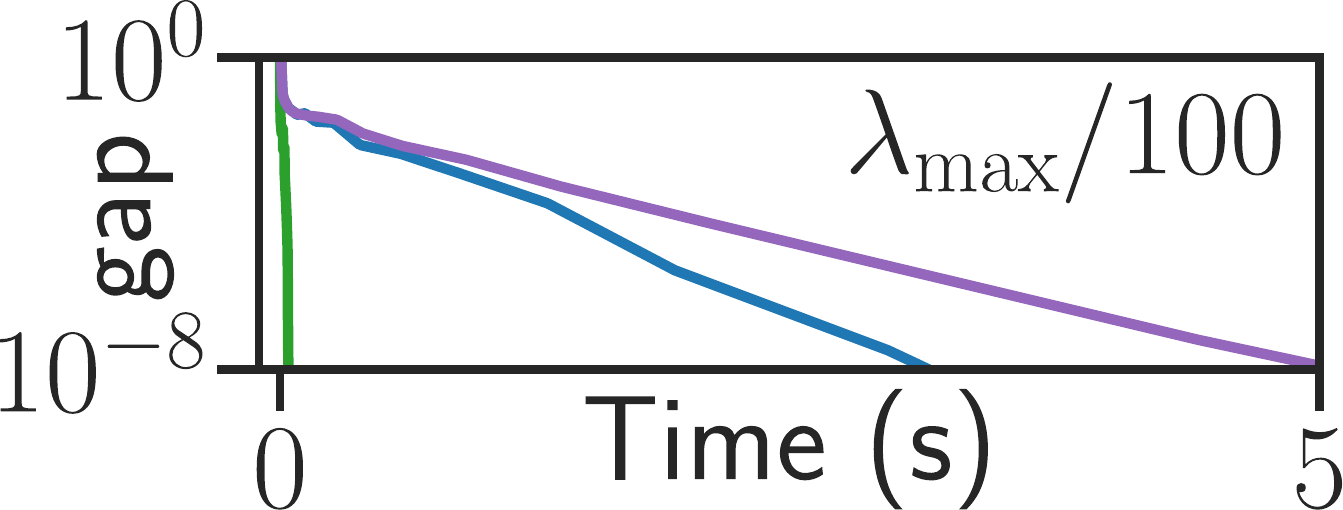}
            % \hfill
            \includegraphics[width=0.2\linewidth]{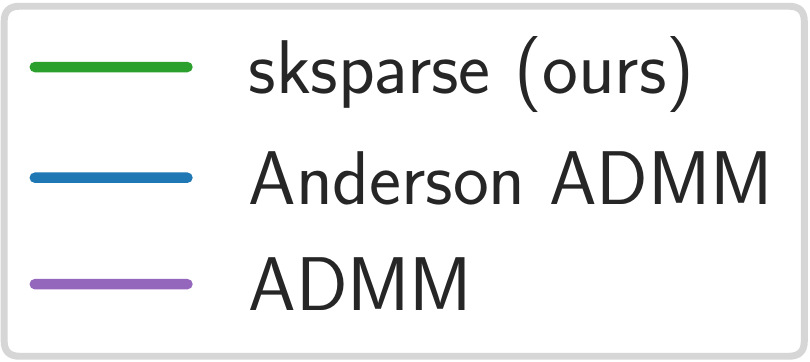}
        \caption{\textbf{ADMM, duality gap.}
        Duality gap as a function of time for the elastic net on a synthetic dataset.}
        \label{fig:duality_gap_admm}
    \end{figure}
%%%%%%%%%%%%%%%%%%%%%%%%%%%%%%%%%%%%%%%%%%%%%%%%%%%%%%%%%%%%%%%
\subsection{\texttt{glmnet} comparison}
%%%%%%%%%%%%%%%%%%%%%%%%%%%%%%%%%%%%%%%%%%%%%%%%%%%%%%%%%%%%%%%
\begin{figure}[H]
    \centering
        \centering
        \includegraphics[width=0.7\linewidth]{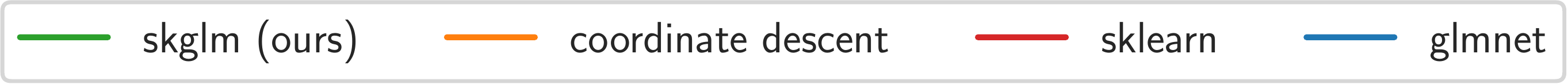}
        \includegraphics[width=\linewidth]{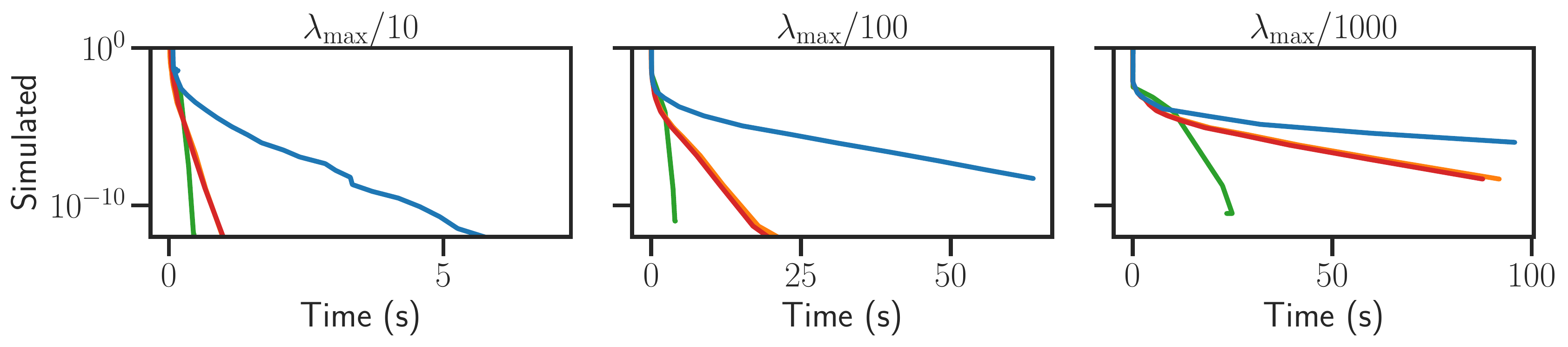}
    \caption{\textbf{Elastic net, duality gap.}
    Duality gap as a function of time for the elastic net on a synthetic dataset, for multiple values of $\lambda$.}
    \label{fig:duality_gap_enet_suppl}
\end{figure}
%%%%%%%%%%%%%%%%%%%%%%%%%%%%%%%%%%%%%%%%%%%%%%%%%%%%%%%%%%%%%%%
\texttt{glmnet} uses a combination of coordinate descent and strong rules to solve the Lasso, elastic net and other L1 + L2 regularized convex problems.
By design of the strong rules \citep{Tibshirani2012}, \texttt{glmnet} is only usable when a sequence of problems must be solved, with decreasing regularization strength $\lambda$: the so-called homotopy/continuation path setting.
In addition, even prompted to solve a given path, the implementation of \texttt{glmnet} does no go up to the smallest $\lambda$ if some statistical criterion stops improving from one $\lambda$ to the other.
Thus, in practice it is nearly impossible to get \texttt{glmnet} to solve a single instance of \Cref{pb:generic_pb} for a given value of $\lambda$.
%
%%%%%%%%%%%%%%%%%%%%%%%%%%%%%%%%%%%%%%%%%%%%%%%%%%%%%%%%%%%%%%%
\subsection{Benchmark on SVM}\label{app:expe_svm}
%%%%%%%%%%%%%%%%%%%%%%%%%%%%%%%%%%%%%%%%%%%%%%%%%%%%%%%%%%%%%%%
Our proposed algorithm can be used with various datafits and penalties.
The SVM primal optimization problem reads
\begin{equation}\label{eq:primal_svm}
    \argmin_{\beta \in \bbR^p} \frac{1}{2} \normin{\beta}^2 + C\sum_{i=1}^n \max(0, 1 - y_iX_{i:}^\top\beta )
    \enspace .
\end{equation}
The dual of \Cref{eq:primal_svm} falls in the framework encompassed by our algorithm, it writes:
\begin{align}\label{pb:dual_svm_one}
    \argmin_{\alpha \in \mathbb{R}^n} & \frac{1}{2} \alpha^\top Q \alpha - \sum_{i=1}^n \alpha_i \\ \nonumber
    &\text{s.t.} 0\leq \alpha_i \leq C \enspace ,
\end{align}
where $Q_{ij} = y_iy_j X_{i:}^\top X_{j:}$.
The datafit is then a quadratic function which we seek to minimize subject to the constraints that $\alpha_i \in [0, C]$.
\Cref{pb:dual_svm_one} is equivalent to the minimization of the following problem:
\begin{problem}\label{pb:dual_svm}
    \argmin_{\alpha\in \bbR^n}\alpha^\top Q \alpha - \sum_{i=1}^n \alpha_i + \iota_{[0, C]}(\alpha_i),
\end{problem}
where $\iota_{[0, C]}$ is the indicator function of the set $[0, C]$.
We also have that the equation link between \Cref{eq:primal_svm} and \Cref{pb:dual_svm} is given by
\begin{align}
    \beta = \sum_{i=1}^n y_i \alpha_i X_{i:} \enspace .
\end{align}

We solved \Cref{pb:dual_svm} on the real dataset \emph{real-sim}.
We compared our algorithm with a coordinate descent approach (CD), the \texttt{scikit-learn} solver based on \texttt{liblinear}, the l-BFGS \citep{Liu_Nocedal1989} algorithm, \texttt{lightning}, and the proposed algorithm (\texttt{skglm}).
\Cref{fig:svm_expe} shows the suboptimality objective value as a function of the time for the different solvers.
The optimization problem was solved for three different regularization values controlled by the parameter $C$ which was set to $0.1$, $1.0$ and $10.0$.
As \Cref{fig:svm_expe} illustrates our algorithm is faster than its counter parts. The difference is larger as the optimization problem is more difficult to solve \ie when $C$ gets large.
%%%%%%%%%%%%%%%%%%%%%%%%%%%%%%%%%%%%%%%%%%%%%%%%%%%%%%%%%%%%%%%
\begin{figure}[H]
    \centering
        \centering
        \includegraphics[width=1\linewidth]{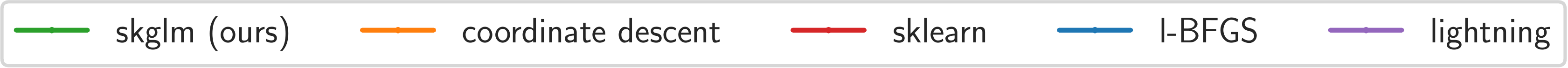}
        \includegraphics[width=\linewidth]{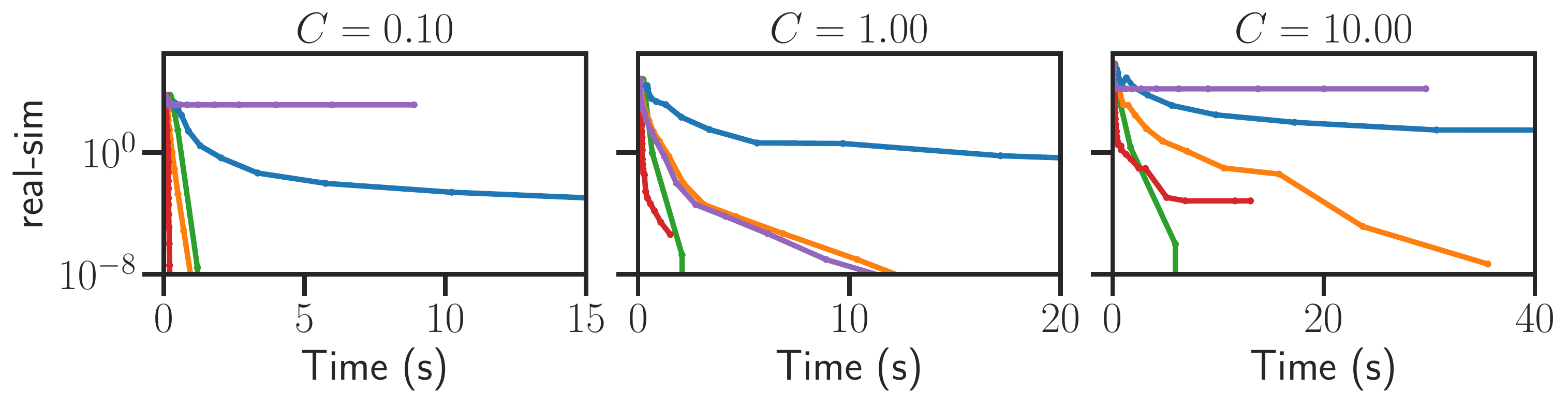}
    \caption{\textbf{SVM, Suboptimality.} Suboptimality as a function of time for the dual SVM optimization \Cref{pb:dual_svm}.
    }
    \label{fig:svm_expe}
\end{figure}
%%%%%%%%%%%%%%%%%%%%%%%%%%%%%%%%%%%%%%%%%%%%%%%%%%%%%%%%%%%%%%%

\subsection{Sparse recovery with non-convex penalties}
\label{app:expe_intro_fig}
In \Cref{fig:sparse_recovery}, to demonstrate the versatility of our approach, we provide a short benchmark of sparse regression, using convex and non-convex penalties.
The data is simulated, with $n=1000$ samples and $p=2000$ features with correlation between features $j$ and $j'$ equal to $0.6^{\abs{j - j'}}$.
The true regression vector $\beta^*$ has $200$ non zero entries, equal to 1.
The observations are equal to $y = X\beta^* + \varepsilon$ where $\varepsilon$ is centered Gaussian noise with variance such that $\normin{X\beta^*} / \normin{\varepsilon} = 5$.
On \Cref{fig:sparse_recovery} we show the regularization path (value of solution found for ever $\lambda$ computed with our algorithm).
Note that despite convergence being only guaranteed towards a local minima, the performance of non-convex estimators is still far better than the global minimizer of the Lasso.
We see that the non-convex penalties are better at recovering the support.
The time to compute the regularization paths is similar for the 4 models, around 1 s.
Thanks to our flexible library, we intend to bring these improvements to practitioners at a large scale.

%%%%%%%%%%%%%%%%%%%%%%%%%%%%%%%%%%%%%%%%%%%%%%%%%%%%%%%%%%%%%
\subsection{Variations in the convergence curves}\label{app:variability}
%%%%%%%%%%%%%%%%%%%%%%%%%%%%%%%%%%%%%%%%%%%%%%%%%%%%%%%%%%%%%
\begin{figure}[h]
    \centering
    \includegraphics[width=0.6\linewidth]{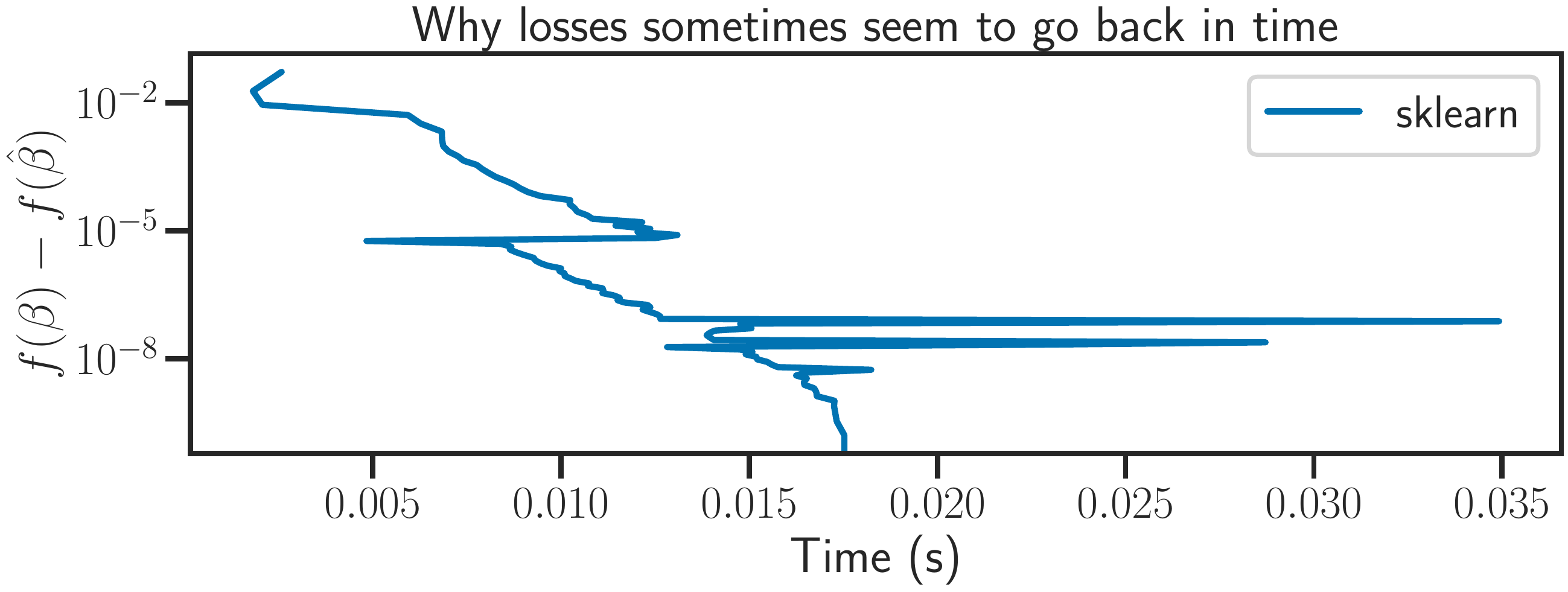}
    \caption{Typical curve aspect caused by variations in solver running time from one run to another (\texttt{scikit-learn}, Lasso problem).}
    \label{fig:non_monotonous_curves}
\end{figure}

By design of the \texttt{benchopt} library that we used for reproducible experiments, solvers are treated as black boxes, for which one only controls the number of iteration performed.
It is thus not possible to monitor the time and losses in a single run of a given solver.
Instead, the solver is run for 1 iteration, then 2 (starting again from 0), then 3, etc.
This allows to obtain a convergence curve for a solver without interfering with its inner mechanisms.
One drawback is that, because of variability in code execution time, it may happen that the run with $K+1$ iterations takes less time than the run with $K$ iterations, for example in \Cref{fig:non_monotonous_curves} -- although it performs more iterations and thus usually decreases the objective more.
Then, the curves seem to go back in time.
The variability can be damped by running the experiment several times and averaging the results, which we did when the total running time allowed it.
Otherwise, these variations should indicate that, as all measurements, convergence curves as a function of time are noisy.

\end{document}